\documentclass{article}

\usepackage[main, preprint]{neurips_2026}

\usepackage[utf8]{inputenc} %
\usepackage[T1]{fontenc}    %
\usepackage{hyperref}       %
\hypersetup{
    colorlinks,
    linkcolor={red!50!black},
    citecolor={blue!50!black},
    urlcolor={blue!80!black}
}
\usepackage{url}            %
\usepackage{booktabs}       %
\usepackage{amsfonts}       %
\usepackage{amsthm}
\usepackage{nicefrac}       %
\usepackage{microtype}      %
\usepackage{xcolor}         %

\usepackage{amsmath} %
\usepackage{mathtools}
\usepackage{cprotect}
\mathtoolsset{showonlyrefs}
\usepackage{amssymb}
\usepackage{bm}
\usepackage{subcaption}
\usepackage{wrapfig}
\usepackage{comment}
\usepackage[dvipsnames]{xcolor}

\newtheorem{theorem}{Theorem}[section]
\newtheorem{proposition}[theorem]{Proposition}

\newtheorem{definition}[theorem]{Definition}

\usepackage{algorithm}
\usepackage[noend]{algpseudocode}

\newcommand{\pdata}{p_{\text{data}}}
\newcommand{\FID}{\text{FID}}
\DeclareMathOperator{\MIND}{\texttt{MIND}}
\newcommand{\E}{\mathbb{E}}
\newcommand{\fP}{\mathbf{P}}
\DeclareMathOperator{\tr}{tr}

\usepackage{enumitem}   
\usepackage{wrapfig}
\usepackage{float}
\usepackage{url}
\usepackage{upgreek}
\usepackage{tikz-cd}

\newcommand{\R}{\mathbb R}

\newcommand{\cN}{\mathcal N}

\newcommand{\cU}{\mathcal{U}}

\newcommand{\cP}{\mathcal{P}}

\renewcommand{\top}{\intercal}

\usepackage{nicefrac}

\title{\texttt{MIND}: Monge Inception Distance\\ for Generative Models Evaluation}

\author{%
  ~~Quentin Berthet$^1$ ~~~~~~Yu-Han Wu$^{1,2}$ ~~~~~~Cl\'ement Crepy$^1$\\[0.5em]
   ~~~~~~~\textbf{Romuald Elie}$^1$ ~~~~~~~~~~\textbf{Klaus Greff}$^1$ ~~~~~~~\textbf{Micha\"el E. Sander}$^1$ \\[1.0em]
  $^1$Google DeepMind\quad$^2$ LPSM, Sorbonne Universit\'e
}

\begin{document}

\maketitle

\begin{abstract}
  We propose the Monge Inception Distance (\texttt{MIND}), a metric for evaluating generative models that addresses key limitations of the widely adopted Fréchet Inception Distance (FID). The \texttt{MIND} metric leverages the sliced Wasserstein distance to compare distributions by averaging one-dimensional optimal transport distances, efficiently computed via sorting. This approach circumvents the estimation of high-dimensional means and covariance matrices, which underlie FID's poor sample complexity and vulnerability to adversarial attacks. We empirically demonstrate three primary advantages: (i) it is more sample-efficient by one order of magnitude, (ii) it is faster to compute by two orders of magnitude, (iii) it is more robust to adversarial attacks such as moment-matching. We show that \texttt{MIND} with 5k samples can replace the evaluation performance of FID with 50k samples, providing high correlation with this standard benchmark and superior discriminative performance. We further demonstrate that even smaller sample sizes (e.g., 1k or 2k) remain highly informative for rapid model iteration.
\end{abstract}

\begin{figure}[ht]
\centering

\includegraphics[width=.9\linewidth]{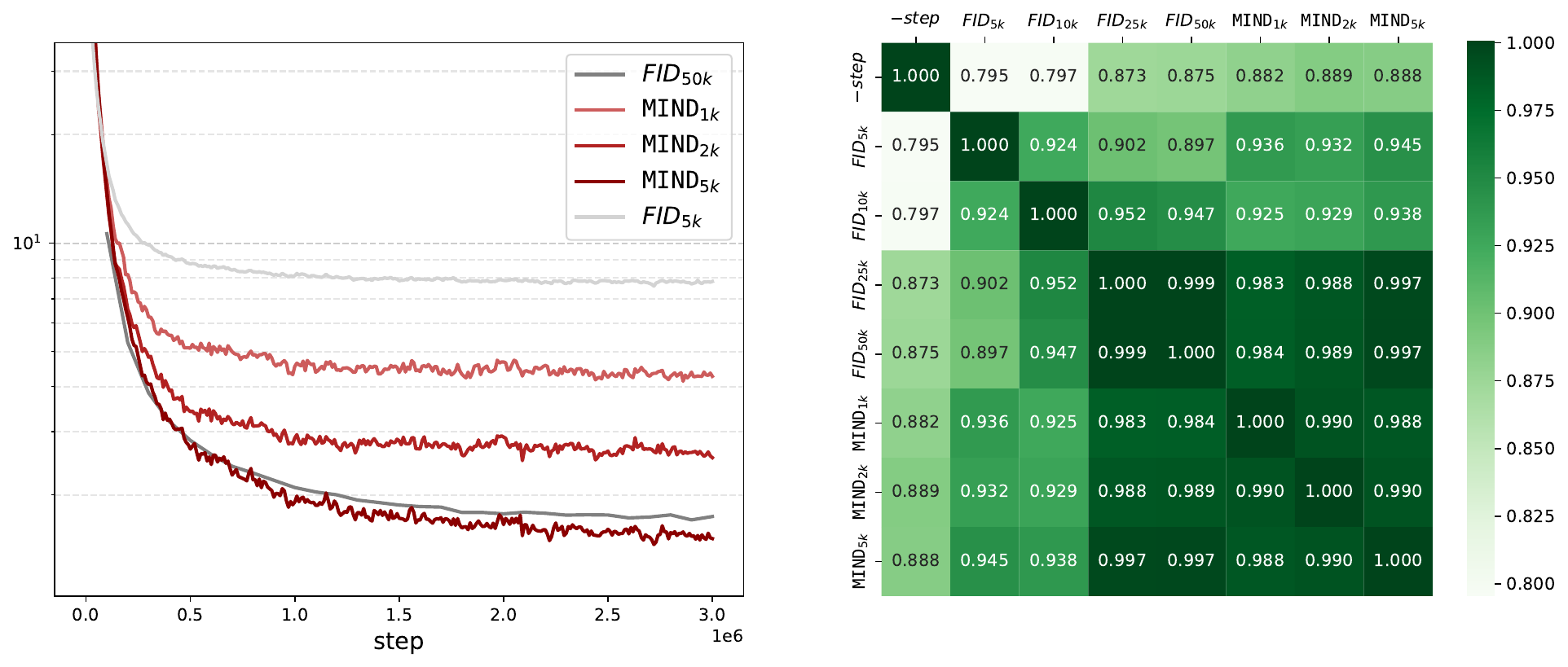}
\caption{(\textbf{Left}) \texttt{MIND} metric during a diffusion model training run on ImageNet-64 (log scale), illustrating how $\MIND_{5k}$ can be used to replace $\FID_{50k}$, with a larger range - see Section~\ref{sec:exp-analysis}. (\textbf{Right}) Correlation with number of training steps - better for $\MIND_{1k}$ and $\MIND_{5k}$ than FID with $50$k samples.
\label{fig:intro}}
\end{figure}

\section{Introduction}

Generative models, especially diffusion models~\citep{diffusion-ddpm}, have set new standards in high-quality data synthesis. This progress has spurred innovation across numerous fields, from creative arts to scientific simulation. However, as models grow in complexity, the metrics used to evaluate them have struggled to keep pace~\citep{stein2023exposing}. The de facto standard, the Fréchet Inception Distance (FID) \citep{heusel2017gans} is based on a Gaussian approximation of pre-trained network embeddings, such as the Inception-v3 model \citep{inception-score}.

This metric relies on estimating high-dimensional mean and covariance matrices from inception embeddings. However, this sample-heavy approach typically requires $50$k samples~\citep{chong2020effectively}, creating a significant development bottleneck. Furthermore, because FID only considers the first two moments of the distributions, it is not a true distance metric and is vulnerable to adversarial "hacking" without corresponding visual improvements \citep{fid-critique-pr}.

In this work, we introduce the Monge Inception Distance ($\MIND$) addressing these limitations. Based on optimal transport theory and named in honor of Gaspard Monge, who introduced the optimal transport problem \citep{monge-1781}, $\MIND$ leverages the sliced Wasserstein distance \citep{rabin2011wasserstein}. Instead of relying on Gaussian simplification as in FID, $\MIND$ reduces the complexity of high-dimensional optimal transport comparison by averaging many one-dimensional projections, where the transport problem is solved exactly via a simple, parallelizable sorting operation. This  captures finer distributional details without the statistical instability inherent in high-dimensional matrix estimation - see \citep{villani2008optimal, peyre2019computational} for a modern perspective on optimal transport theory and applications. 

We highlight that this approach yields stable, high-quality evaluation using only $5$k samples -- an order of magnitude fewer than FID -- while offering $100\times$ faster computation and increased robustness to adversarial moment-matching. We statistically validate the performance of $\MIND$ across various hypothesis testing problems at different sample sizes. Finally, while we present our results using Inception-v3 embeddings to facilitate a direct comparison with the current FID benchmark, the $\MIND$ metric is fundamentally embedding-agnostic. It is modality-independent and can be seamlessly applied to any representation space, including CLIP \citep{radford2021learning}, DINO~\citep{oquab2024dinov2,simeoni2025dinov3}, or specialized embeddings for audio and video synthesis.

\textbf{Main contributions.} In this work, we introduce $\MIND$, a metric for improved evaluation of generative models. We demonstrate the following advantages of this metric:
\begin{itemize}[topsep=0pt,itemsep=2pt,parsep=2pt,leftmargin=10pt]
\item[-] \textbf{Sample Efficiency}: We show that $\MIND_{5k}$ provides a stable evaluation that correlates highly with $\FID_{50k}$, enabling reliable benchmarking with $10\times$ fewer samples.

\item[-] \textbf{Computational Speed and Memory Efficiency}: Due to its reliance on 1D sorting rather than high-dimensional matrix operations, $\MIND$ is over $100\times$ faster to compute, and requires $10\times$ less memory, facilitating real-time evaluation during training.

\item[-] \textbf{Metric Robustness}: Since $\MIND$ is derived from a proper distance, we show that it is significantly more resistant to "metric hacking" via moment-matching attacks that can artificially lower FID.

\item[-] \textbf{Discriminative Power}: Our experiments show that $\MIND$ more reliably distinguishes between model checkpoints and identifies subtle image perturbations at low sample sizes.
\end{itemize}

\section{Generative model evaluations}

We consider the problem of evaluating a generative model $g_\theta$, that generates outputs by mapping noise $Z \sim \cN(0, I)$ to data outputs (e.g. images) $a = g_\theta(Z)$. In standard practice, these outputs are passed through a pre-trained feature extraction model $\psi_{w}$ to obtain embeddings. For an Inception-v3 model \citep{szegedy2016rethinking}, these embeddings typically reside in dimension $d=2,048$.

The performance of the model is measured by the statistical distance between the distribution of generated embeddings, $X = \psi_{w}(g_{\theta}(Z)) \sim p_{\theta}$, and the distribution of real dataset embeddings, $Y = \psi_{w}(D) \sim p_{data}$. In practice, this distance is estimated using finite samples of size $n$, denoted as the empirical distributions $\hat{p}_{n,\theta}$ and $\hat{p}_{n,data}$ (see Appendix~\ref{app:defs}). Any measure of statistical distance between these distributions can be used, and we consider in this work several inception distances, defined as follows for any distance or divergence $\Delta$ between distributions \citep[see, e.g.][]{cover1999elements}.

\begin{definition}[General - Inception distance]
\label{def:distance}
    Let $X = \psi_w(g_\theta(Z)) \sim p_\theta$, $Y = \psi_w(D) \sim \pdata$.  
    For a distribution distance function $\Delta$, the performance of the model $g_\theta$ is given by $\Delta\text{ID}(p_\theta, \pdata)$.
    With a sample of size $n$, its empirical estimate is the plug-in value $\Delta\text{ID}(\hat p_{n, \theta}, \hat p_{n, \text{data}})$.
\end{definition}
\begin{figure}[h]
\begin{center}
\includegraphics[width=\linewidth]{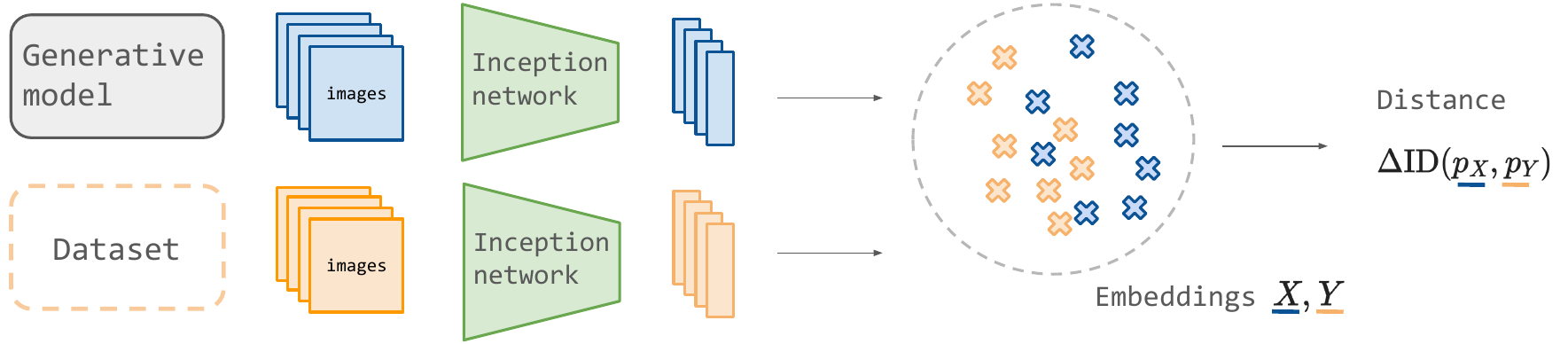}
\caption{General pipeline for evaluating generative model sampling distance to a dataset.\label{fig:pipeline}}
\end{center}
\end{figure}

\subsection{Existing method: Fr\'echet Inception Distance - FID}
FID is the most widely adopted instance of an inception distance. It measures the distance between two distributions based on their first two moments: the mean ($\mu$) and covariance ($\Sigma$), using the squared $2$-Wasserstein distance $W_2^2$ (see Appendix~\ref{app:ot})

\begin{definition}[Fr\'echet Inception Distance - FID]
Let $X = \psi_w(g_\theta(Z)) \sim p_\theta$ and $Y = \psi_w(D) \sim \pdata$ and $\mu_X, \Sigma_X$ and $\mu_Y, \Sigma_Y$ be the means and covariances of $p_{\theta}$ and $p_{data}$, respectively.
The FID is defined as the squared 2-Wasserstein distance between two fitted Gaussians: %
\[
\FID(p_\theta, \pdata) =  \|\mu_X - \mu_Y\|^2  + \tr(\Sigma_X + \Sigma_Y - 2(\Sigma_Y\Sigma_X)^{1/2})\, .
\]
In practice, this is computed using empirical sample means and covariances $\hat{\mu}_n$ and $\hat{\Sigma}_n$.
\end{definition}
This approach was originally motivated as a way to bypass the high computational and statistical complexity of a direct, sample-based Wasserstein distance by using a Gaussian approximation.

\paragraph{Drawbacks}    
Despite its widespread use and role as de facto standard metric, there are several drawbacks in using this distance, as noted in several works \citep[see, e.g.][]{karras2017progressive,stein2023exposing, jayasumana2024rethinking, bischoff2024practical, yang2026representation}
\begin{itemize}[topsep=0pt,itemsep=2pt,parsep=2pt,leftmargin=10pt]
    \item[-] Computing this distance is based on the estimate $\hat \Sigma_n$ of a $d$-by-$d$ covariance. It is rank-deficient for $n \le d$, creating numerical and statistical issues when estimating the second term, unless the sample size is at least of order $d$, which is $2048$ for inception networks. For images, this means that the sample sizes usually used are $10$k or $50$k~\citep{binkowski2018demystifying,chong2020effectively}, with a high impact on evaluation time and cost.
    
    \item[-] FID is also not a proper distance~\citep{jayasumana2024rethinking}: two distributions can have the same mean and covariance and be very different~\citep[see, for example,][Section 30]{billingsley2017probability}. We show that, as a consequence, it is not robust, and this fact can be leveraged to artificially reduce the FID without visually altering the images (see Section~\ref{sec:exp-hacking}).
\end{itemize} 

We also evaluate other inception-based distances proposed as metrics, as means of comparison, such as the {\em Maximum mean discrepancy} (MMD), or Sinkhorn divergence. We provide full definitions and a discussion in appendix (Section~\ref{app:metrics}), and include them in comparisons.

\section{Proposed method: Monge Inception Distance - \texttt{MIND}}
We propose the following metric to overcome these challenges, based on the sliced Wasserstein distance which has several known advantages, both in terms of statistical and computational complexity. It is an average of the Wasserstein distances of the distribution of projections, over all unit directions \citep[see, e.g.][and Appendix~\ref{app:ot} in this work for details]{rabin2011wasserstein, nadjahi:tel-03533097}.

The sliced Wasserstein distance is a proper distance - it is equal to $0$ if and only if both distributions are equal \citep[][Proposition 5.1.2]{bonnotte2013unidimensional}. We use this approach for our \texttt{MIND} metric, taking the average of Wasserstein distances projected along finitely many unit directions for an estimate.

\begin{definition}[Monge Inception Distance]
\label{def:mind}
Let $X = \psi_w(g_\theta(Z)) \sim p_\theta$ and \mbox{$Y = \psi_w(D) \sim \pdata$}, and $\mathcal{U}(S)$ be the uniform distribution on the unit sphere. $\MIND$ is given by averaging $W^2_2$ distances for projections of the distributions along unit directions, with a multiplicative scaling $\alpha=3d$
\[
\MIND(p_\theta, \pdata) = \alpha \E_{u\sim\cU(S)} [W^2_2(u^\top p_\theta, u^\top \pdata)]\, ,
\]
where $u^\top p_\theta$ (resp. $u^\top \pdata$) is the distribution of $u^\top X$ when $X \sim p_\theta$ (resp. of $u^\top Y$ when $Y \sim \pdata$) and $d$ is the data dimension.

For finite samples $(X_j)_{j\in[n]}$, $(Y_j)_{j\in[n]}$, random unit directions  $(u_i)_{i \in [M]}$ and $\alpha=3d$, it is given by
\[
\MIND(\hat p_{n, \theta}, \hat p_{n, \text{data}}) = \frac{\alpha}{M} \sum_{i=1}^M W^2_2(u_i^\top \hat p_{n, \theta}, u_i^\top \hat p_{n, \text{data}}) = \frac{\alpha}{nM}\sum_{i=1}^M \sum_{j=1}^n |\text{sort}(u_i^\top X)_j -  \text{sort}(u_i^\top Y)_j|^2\, .
\]
\end{definition}

\paragraph{Remarks}
\begin{wrapfigure}[21]{r}{0.6\textwidth}
    \centering
    \includegraphics[width=\linewidth]{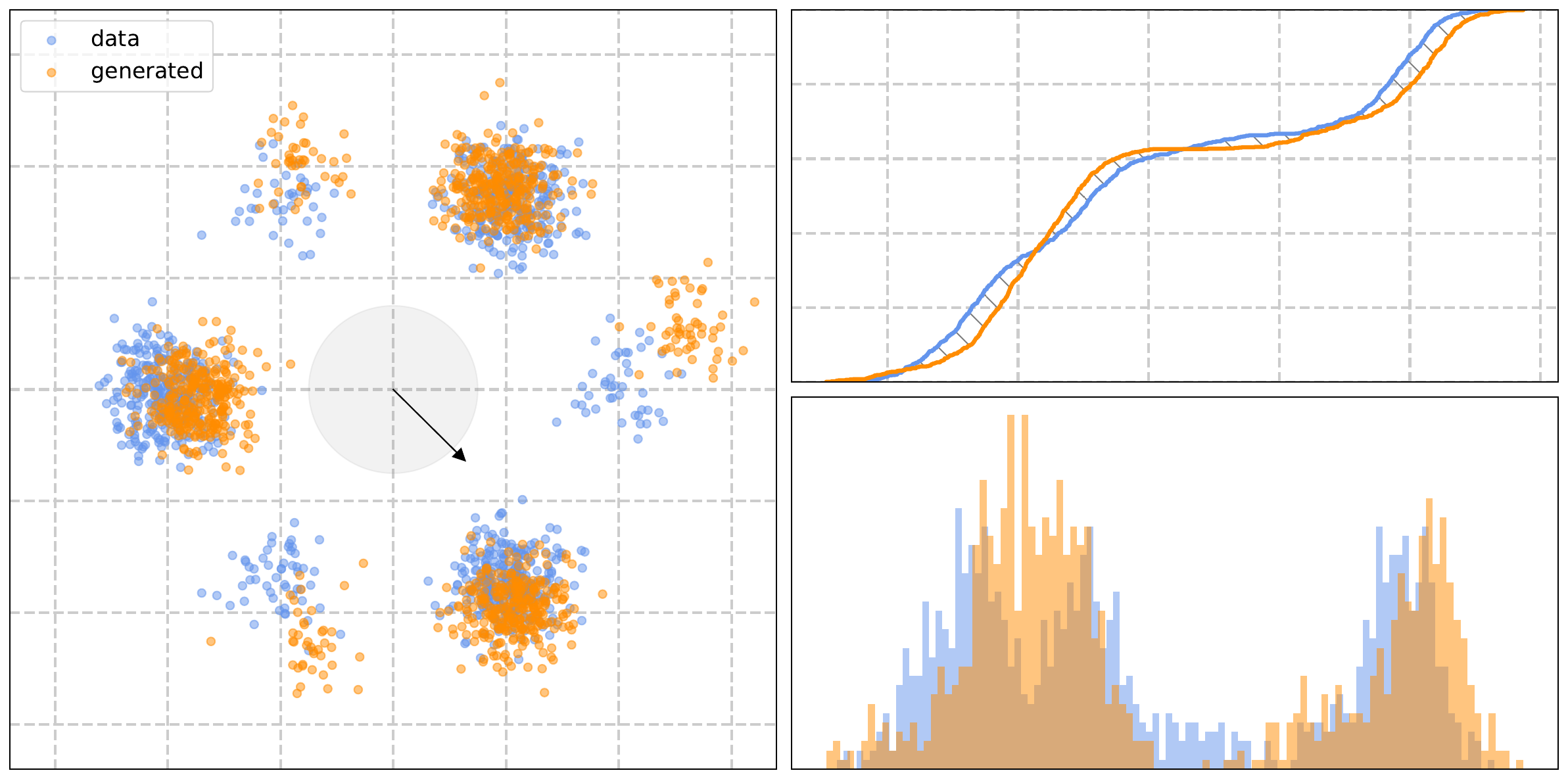}
    
    \cprotect\caption{Computation of \texttt{MIND} based on the idea of Sliced Wasserstein, illustrated in 2D with a single projection. (\textbf{Left}) Two samples of synthetic embeddings (orange and blue), along with the unit sphere and a random unit direction $u$. 
    (\textbf{Bottom Right}) The two histograms of distributions of the projections along $u$.
    (\textbf{Top Right}) The associated cumulative distribution functions (cdf), the hatched area is related to 1D Wasserstein distances along $u$: it is the $W_1$ distance, and used for all convex costs with pairwise sorted distances.}
\end{wrapfigure}
\begin{itemize}[topsep=0pt,itemsep=2pt,parsep=2pt,leftmargin=10pt]
\item[-] \texttt{MIND} relies on two finite sample estimates: the $1$D Wasserstein distance over samples $X_j$ and $Y_j$ of size $n$, and the expectation  $\E_{u\sim \cU(S)}$ over $M$ random unit directions $u_i$.

\item[-] Although we adopt the name of Inception Distance for consistency with established literature, the formulation of \texttt{MIND} does not depend on the Inception architecture. The mapping function $\psi_w$ can represent any feature extractor; consequently, \texttt{MIND} serves as a general-purpose tool for evaluating distributional similarity across diverse data modalities and embedding models.

\item[-] This $1$D formulation allows a more stable evaluation using an order of magnitude fewer samples -- with $n$ of order $5$k rather than $50$k (see Sections~\ref{sec:exp-analysis} and~\ref{sec:exp-sample}). Furthermore, because the sliced Wasserstein distance is a proper distance, having matching means and covariance matrices is not sufficient for $\MIND$ to be zero, making it inherently robust to moment-matching hacking (see Section \ref{sec:exp-hacking}). 

\end{itemize}
\begin{itemize}[topsep=0pt,itemsep=2pt,parsep=2pt,leftmargin=10pt]
\item[-] Leveraging the exact solution of $1$D transport problem \citep[see, e.g.][and Appendix~\ref{app:ot} in this work]{peyre2019computational}, the distance simplifies to pair-wise difference between sorted elements:
\[
W_2^2(\hat p_n, \hat q_n) = \frac{1}{n} \sum_{j=1}^n |\text{sort}(x)_j - \text{sort}(y)_j|^2\,,
\]
where $\text{sort}: \R^n \to \R^n$ is the function that maps a vector $x\in \R^n$ to its copy sorted in nondecreasing order, i.e. $\text{sort}(x) = (x_{\sigma(1)},\ldots,x_{\sigma(n)})^\top$ such that $x_{\sigma(1)} \le \ldots \le x_{\sigma(n)}$ (ties do not make this function ambiguous). Since the sorting operation runs in $O(n\log n)$ time, $\MIND$ avoids the need to estimate or store high-dimensional objects (e.g. $d\times d$ matrices for FID, $n\times n$ matrices for other distances). Similarly to FID, we use the square of the distance rather than taking a square root of the average. We also use a multiplicative scaling factor $\alpha$---see discussion in Section~\ref{sec:hyp}.

\end{itemize}
\begin{figure}

\hspace{3em}
\begin{subfigure}{.45\textwidth}
\subcaption{JAX}
 \centering
\fbox{
\includegraphics[width=\linewidth]{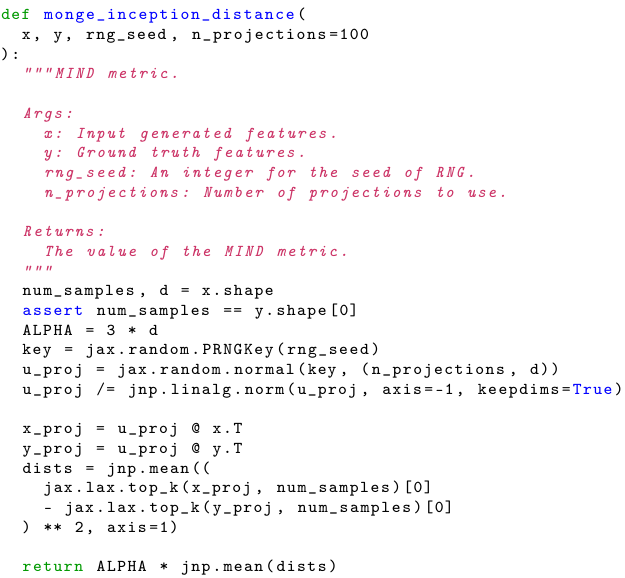}
}
\end{subfigure}
\hfill
\begin{subfigure}{.345\textwidth}
 \centering
 \subcaption{PyTorch}
\fbox{
\includegraphics[width=\linewidth]{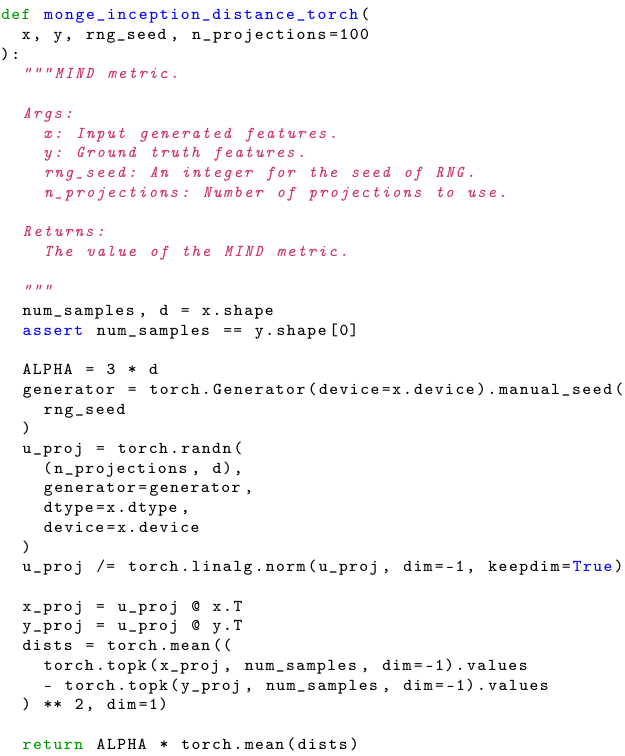}
} 
\end{subfigure}
\hspace{3em}
\cprotect \caption{JAX (a) and PyTorch (b) implementation of \texttt{MIND}.}
\label{fig:implementation}
\end{figure}

\section{Experiments}
\subsection{Implementation}
As noted above, estimating $\MIND$ on two samples of size $n$ is both computationally and conceptually easy. It requires only trivially parallelizable projections and sorting operations. As such, it is particularly adapted to modern accelerated-oriented hardware and software. We provide in  Figure~\ref{fig:implementation} both a JAX \citep{bradbury2018jax} and PyTorch \citep{paszke2019PyTorch} implementation, in the form of a short code snippet that can be directly used in an evaluation pipeline. We also provide in Section~\ref{sec:computation} and~\ref{sec:memory} experimental results showcasing the computation time and memory advantages of this algorithm compared to other methods.

\subsection{Hyperparameter choices}
\label{sec:hyp}

As noted in Definition~\ref{def:mind}, we scale the $\MIND$ metric by a multiplicative factor $\alpha > 0$. This is done so that the order of magnitude of this metric matches those of FID. This proximity helps to compare values of $\MIND$ to those of FID, and is chosen to favor adoption. Based on an analysis on ImageNet-64, we have found that taking $\alpha = 3 \times d \approx 6,000$ is a good fit, especially later in a training run (where $d=2,048$ is the dimension of the embedding space) - see Figure~\ref{fig:train_fid_mind}.  We also observe that $\MIND_{5k}$ has more range than FID for any sample size, with higher values early in the run (above even those of $\FID_{5k}$), and aligned with $\FID_{50k}$ later in the run, and better aligned with the number of steps in a training run (see Figure~\ref{fig:intro}, right). We also compute $\MIND$ and FID score with features of various dimension. The features are obtained by truncating the Inception-v3 features to the target dimension. Figure~\ref{fig:train_fid_mind} (Right) shows that the $\MIND$ remains an affine relation with respect to FID while varying the dimension of the feature space in the $\log$-$\log$ plot---justifying the choice of the scaling factor $\alpha\propto d$.

\begin{figure*}[ht]
  \centering
  \begin{subfigure}{.47\textwidth}
  \centering
  \includegraphics[width=\linewidth]{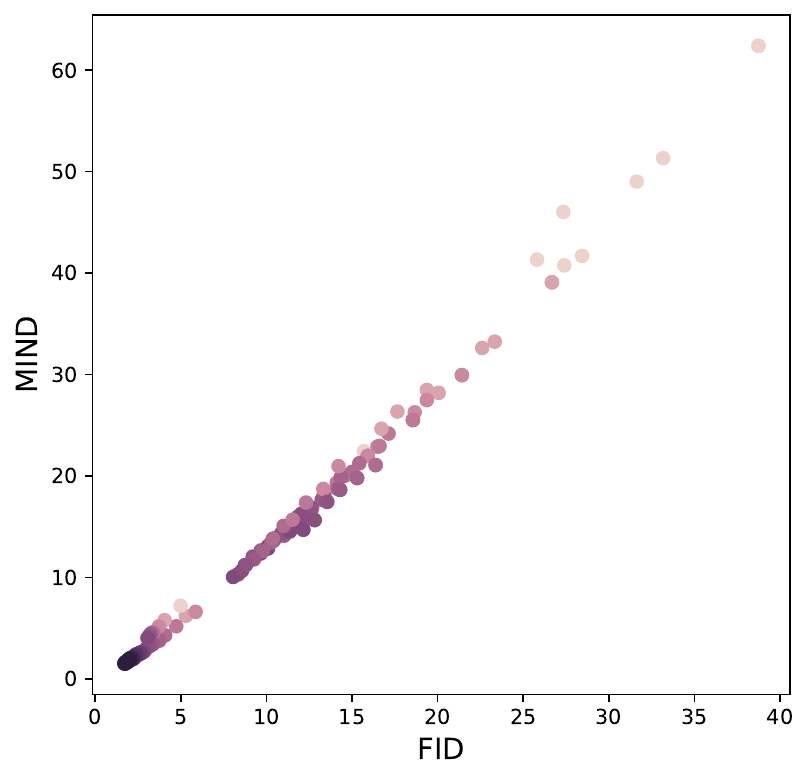}
  \end{subfigure}
  \begin{subfigure}{.48\textwidth}
  \centering
  \includegraphics[width=\linewidth]{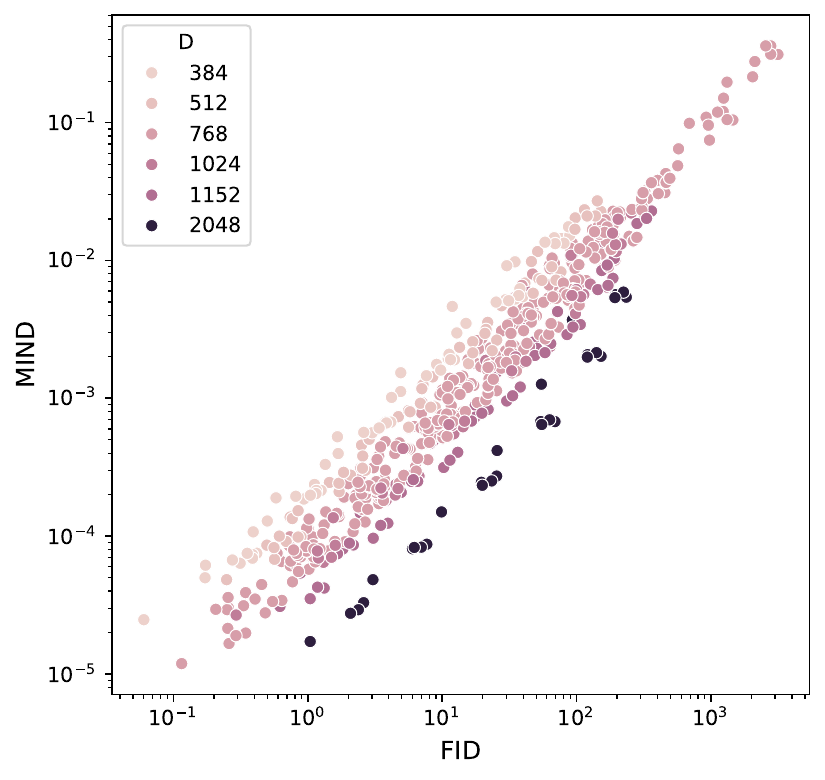}
  \end{subfigure}
    \caption{$\MIND_{5k}$ and $\FID_{50k}$ (and other sample sizes) for a model at different steps of training on ImageNet-64. \textbf{Left}: All $\MIND$ metrics are rescaled by a factor $\alpha \approx 6,000$ chosen to optimize proximity with FID, colors indicating the step at which the metric is evaluated. \textbf{Right}: $\MIND$ without scaling factor for various embedding dimension, colors indicating the dimension in which the metric is evaluated.}
  \label{fig:train_fid_mind} %
\end{figure*}
\subsection{\texttt{MIND} analysis}
\label{sec:exp-analysis}
We illustrate the behavior of the \texttt{MIND} metric by analyzing its dependency on $n$ (the number of samples from the distributions) and $M$ (the number of random projections). Fixing a number $k>0$ (varies in different settings), we evaluate its ability to correctly order $k$ different distributions $p^1,\ldots p^k$ relative to the true data distribution. 
Specifically, we calculate the probability of error in correctly ranking the sequence of distances $\MIND(\hat p_n^1, \hat p_{n, \text{data}}), \ldots, \MIND(\hat p_n^k, \hat p_{n, \text{data}})$. This measures the ability of the metric to distinguish different images from the elements of the dataset. Our observations indicate that $n = 5,000$ samples is sufficient to reliably distinguish these distributions (we benchmark this against other metrics quantitatively in Section~\ref{sec:exp-sample}).

\begin{figure*}[t]
\centering
\includegraphics[width=\textwidth]{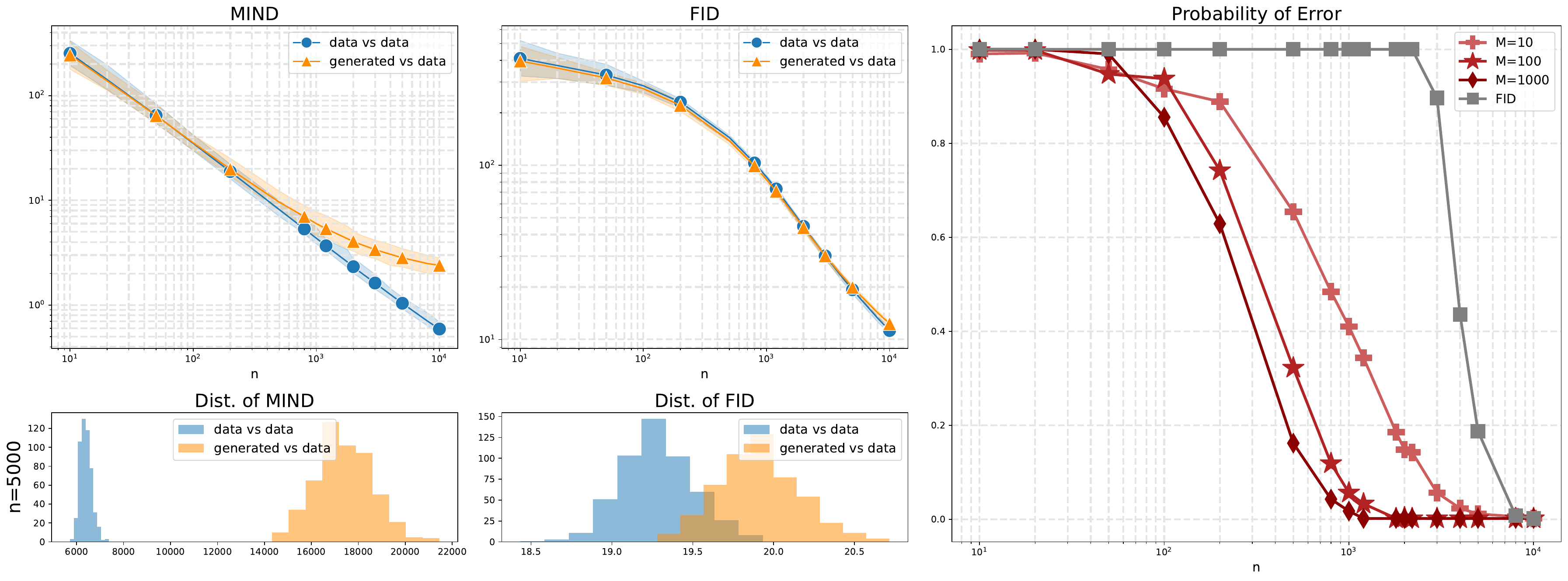}
\caption{(\textbf{Top Left and Middle}) Behavior of the \texttt{MIND} and FID metric in $n$, to distinguish true images from the dataset (base - in blue) from generated images (model - in orange). (\textbf{Bottom Left and Middle}) Histogram of the trials for $n=5,000$ - A bigger gap is better. (\textbf{Right}) Probability of error defined in Section~\ref{sec:gen-vs-true} for three values of $M \in \{10, 100, 1000\}$.}
\label{fig:overlaps_all_m}
\end{figure*}

We also plot the dependency of the estimated metric on $M$, the number of uniformly chosen random projections. This dependency is easier to analyze, since the Monte-Carlo estimate is obtained by averaging unbiased terms to compute the expected Wasserstein distance. We observe that choosing $M$ in the range $[100,\, 1000]$ is sufficient (see Appendix~\ref{app:results}).

\subsection{Metric comparison}
\label{sec:exp-sample}
\begin{wrapfigure}[19]{r}{0.5\textwidth}
    \centering
    \vspace{-0.6cm}
    \includegraphics[width=0.9\linewidth]{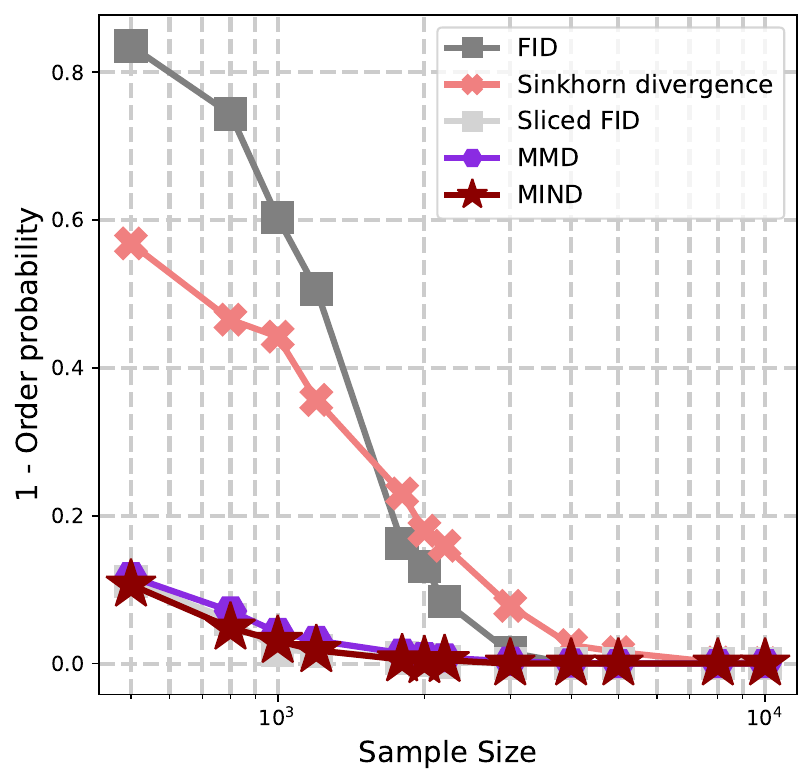}
    \caption{Sample complexity measured by the probability of error for the correct order at five different steps of training.    \label{fig:order}}
\end{wrapfigure}

Running evaluations during training of a diffusion model, we observe that instead of using $\FID_{50k}$ (commonly used post-training because of the cost and time associated with the high sample size), we can use $\MIND_{5k}$ (we evaluate their precisions more quantitatively in the rest of this section). As visible in Figure~\ref{fig:train_fid_mind}, these two metrics are highly correlated, especially later during training.

In order to compare different metrics in a principled fashion, we evaluate how useful they are to distinguish distributions. This provides natural {\em tasks} for which these metrics can be evaluated as {\em methods}, through the lens of statistical hypothesis testing. This can also be understood as a comparison in distribution of the metrics (for random samples), rather than a single value. Given sample size $n$ and metric $\Delta$, we state our three statistical hypothesis testings in the following.

\subsubsection{Generated vs. true data}
\label{sec:gen-vs-true}

This is done by comparing two distributions $p_\theta$ (for some pre-trained model $g_\theta$, with parameters $\theta$) with $\pdata$, and estimating $\Delta(\hat p_{n, \theta}, \hat p_{n, \text{data}})$. Our diffusion model utilizes a U-Net backbone~\citep{ronneberger2015u,nichol2021improved} trained on Imagenet-64. Experiments relying on real data use a fixed set of 100,000 original Imagenet-64 images. Conversely, for evaluations involving data generated from models, we generate a dedicated fixed set of 50,000 samples from each model under evaluation.

We are comparing the values of the metric under two settings, where $\hat p_{n, \theta}$ is a distribution of $n$ samples generated from a trained model, and $\hat p_{n, \text{data}}$ and $\hat p'_{n, \text{data}}$ are two independent samples of size $n$ from the data. The probability of error is defined as:
\[
    \fP\big(\Delta(\hat p_{n, \text{data}}, \hat p'_{n, \text{data}})\ge \Delta(\hat p_{n, \text{data}}, \hat p_{n, \theta})\big)\, .
\]
This measures the ability of the metric to distinguish generated images from elements of the dataset. The results are given in Figure~\ref{fig:overlaps_all_m}. We remark that, $\MIND$ is able to separate the two distributions as soon as $n\ge 5,000$ (Figure~\ref{fig:overlaps_all_m} Left column), while FID requires more than $10$k samples to do so (Figure~\ref{fig:overlaps_all_m} middle column).
\subsubsection{Monotonicity}

We compare $\MIND$ with other metrics using a diffusion model $g_\theta$ trained on the ImageNet-64 dataset \citep{deng2009imagenet}, from which we selected five models, $g_{\theta_1},\dots,g_{\theta_5}$, corresponding to five distinct training checkpoints and generate $50$k images with each of them. The probability of error in ranking these checkpoints correctly is:
    \[
    1 - \fP\big(\Delta(\hat p_{n, \text{data}}, \hat p_{n, \theta_1})\ge \ldots\ge \Delta(\hat p_{n, \text{data}}, \hat p_{n, \theta_k})\big)\, .
    \]
For several sample sizes $n$ ranging from $10$ to $10$k, we perform $512$ independent trials for each metric. We observe that \texttt{MIND}, MMD and sliced FID achieves similar performance in this test (Figure \ref{fig:order}).

\subsubsection{Perturbations}
We consider three different types of image perturbations, the severity of perturbation is given by a parameter $\varepsilon$. Each experiment is performed with $512$ independent trials. We measure the performance of each metric using
\begin{align*}
1 - \fP\big(\Delta(\hat p_{n, \text{data}}, \hat p_{n, \text{data}, \varepsilon_1}) &\le\ldots\le  \Delta(\hat p_{n, \text{data}}, \hat p_{n, \text{data}, \varepsilon_k})\big)\, .
\end{align*}
which is the probability of failing to order all perturbation levels. The results are summarized in Figure~\ref{fig:pert-mixture}. We highlight that, for $n\ge 5,000$, $\MIND$ achieves the same level of performance as MMD while FID is worse on all tasks.
\paragraph{Gaussian blur.} We select a perturbation level \mbox{$\varepsilon\in\{0.2, 0.4, 0.6, 0.8, 1.0\}$}. We apply a Gaussian filter with standard deviation $\varepsilon$ to each image in ImageNet-64.
\paragraph{Rectangle.}We select a perturbation level \mbox{$\varepsilon\in\{0.05, 0.1, 0.15, 0.2\}$}. We randomly place $5$ squares of size $10\times10$ with $\varepsilon$ opacity in each image of ImageNet-64.
\paragraph{Mixture of datasets.} We select a perturbation level \mbox{$\varepsilon\in\{1\%, 3\%, 5\%, 7\%, 10\%\}$}. We draw samples with proportion of $(1-\varepsilon)$ from ImageNet-64 and of $\varepsilon$ from CelebA \citep{liu2015faceattributes}.

\begin{figure*}[ht]
    \centering
    \includegraphics[width=\linewidth]{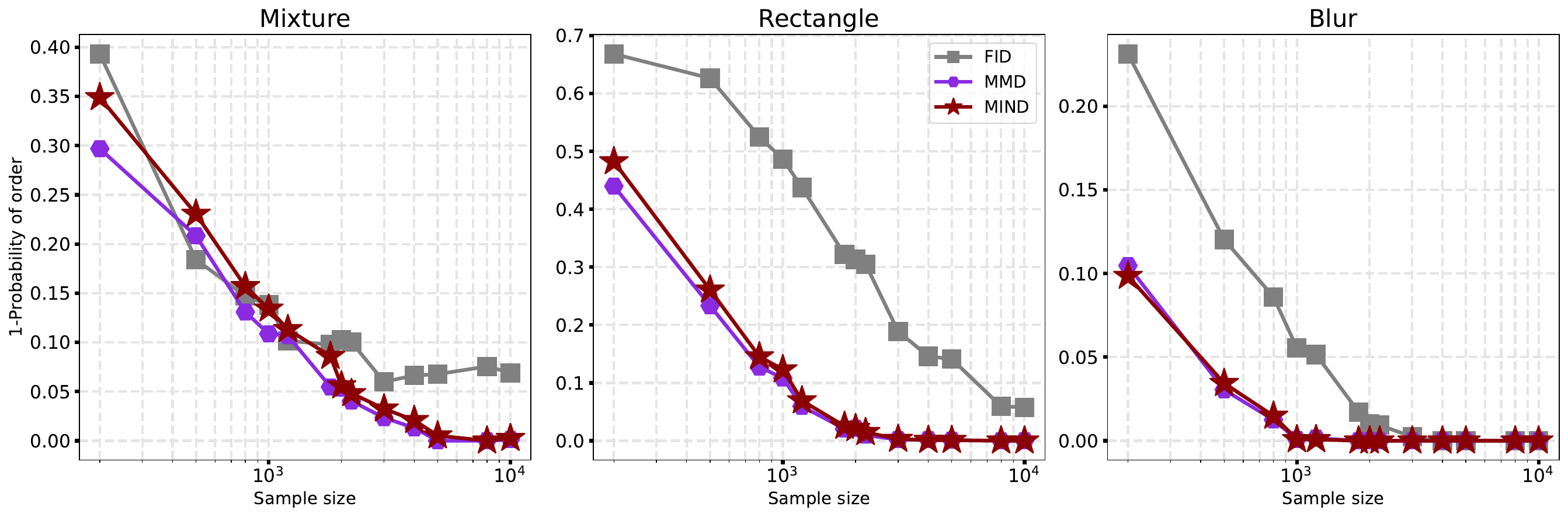}
    \caption{Sample complexity measured by the probability of detecting a small perturbation.}
    \label{fig:pert-mixture}
\end{figure*}

\subsection{Robustness to metric hacking with moment matching}
\label{sec:exp-hacking}
As mentioned above, one of the weaknesses of FID is that it is not a proper distance. Indeed, since this metric is only a function of the means and covariances of the considered distributions, if $p$ and $q$ share the same first and second moments, then $\FID(p, q) = 0$. This is not the case for proper metrics derving from distances such as \texttt{MIND}. We leverage this fact to create an artificial distribution of samples that have a desired mean and covariance. This construction is based on the following property whose proof is in Appendix~\ref{app:proof}. %
\begin{proposition}
\label{pro:target}
    Let $p$ be a target distribution over $\R^d$ with mean $\mu$ and covariance $\Sigma$, whose eigendecomposition is
    \[
    \Sigma = USU^\top =  \sum_{i=1}^r \lambda_i u_iu_i^\top\, .
    \]
Define the $2r$ vectors $v^{(+)}_i$ and $v_i^{(-)}$ indexed by $i \in [r]$, by
\[
v^{(+)}_i = \mu + \alpha u_i\, , \quad v^{(-)}_i = \mu - \alpha u_i\, ,
\]
with $\alpha = \sqrt{\textbf{Tr}(\Sigma)}$. Define $\pi^{(+)}_i = \pi_i^{(-)} = \lambda_i / (2\textbf{Tr}(\Sigma))$ and note that the $\pi_i$ are nonnegative and sum to 1. Let $\hat q$ be the distribution of the $v_i$'s, each with probability $\pi_i$, given by 
\[
    \hat q = \sum_{i=1}^r\pi_i^{(+)} \delta_{v_i^{(+)}} + \sum_{i=1}^r\pi_i^{(-)} \delta_{v_i^{(-)}}\,.
\]

It holds that
\[
\E_{\hat q}[v] = \mu\, , \quad \E_{\hat q}[(v - \mu)(v-\mu)^\top] = \Sigma\, , \quad \FID(\hat q, p) =0\, .
\]
\end{proposition}

\subsection{Moment matching procedure}
This proposition can be leveraged to perform {\em metric hacking} with moment matching: For a batch of $n$ images $a^0$ with embedding distribution $\hat q^{0}$ we construct $a = a^0 + \varepsilon$ whose embeddings have distribution $\hat q$ such that the metric $\Delta(\hat q, \hat p_{n, \text{data}})$ is much smaller than  $\Delta(\hat q^{0}, \hat p_{n, \text{data}})$ (for some data distribution $\pdata$) by optimizing the moments of these embeddings, using the target vectors given by Proposition~\ref{pro:target} with no visually discernible alteration (see Figure~\ref{fig:hack-vis} in Appendix~\ref{app:hack}).

We do so in the following manner: for $n = 2r$ and a batch $a^{0}\in \R^{2r \times [\text{dims}]}$ of $2r$ images, each of shape $[\text{dims}]$ (e.g. $[512, 512, 3]$), and a target distribution $\pdata$ over $\R^d$,  we consider the following objective, aiming to give each $a_i$ an embedding close to the target $v_i$
\[
\min_{a \in \R^{2r \times [\text{dims}]}} \ell (a) = \min_{a \in \R^{2r \times [\text{dims}]}} \sum_{i=1}^{2r} \|\psi_w(a_i) - v_i \|^2\,.
\]

\begin{wraptable}[9]{r}{0.35\textwidth}
\vspace{-.95cm}
    \centering
    \begin{tabular}{lc}
    \toprule
    \textbf{Metric} & ratio\\
    \midrule
    \textbf{FID}  & 11.2\% \\
    \textbf{$\mu \FID$}  &  2.6\% \\
    \textbf{$\sigma \FID$}  & 4.2\% \\
    \textbf{MMD}  & 12.2\% \\\\
    \textbf{\underline{\texttt{MIND}}}  &  31.1\% \\
    \bottomrule
    \end{tabular}
    \caption{Robustness of several metrics under moment matching}
    \label{tab:robustness}
\end{wraptable}
We initialize $a$ at $a^0$ that is $2r$ copies of the same image. This optimization problem is highly parallelizable since the loss is fully separable over each of the $a_i$, and we can use stochastic based optimization methods to solve it. If the batch $a$ satisfy $\ell(a) = 0$, then the FID of the distribution of the $a_i$ with probabilities $
\pi_i$ is also $0$, and we show that optimizing this loss reduces the FID significantly.

In this experiment, we use a full-rank batch, $r=2048$, the dimension of the latent space. Therefore, the total batch size is $n=4096$ and we separate the optimization problem. We use in our evaluation $M=1000$ for \texttt{MIND} and $50$k to compute the reference mean and covariance for the FID. The results summarized in Table 1 show that several of these metrics are highly sensitive to moment matching hacking, with only $~10\%$ or less of the metric remaining for the baseline metrics (and much less for sample-efficient metric versions of the FID), and that while affected, the \texttt{MIND} metric is much more robust.

\subsection{Computation time comparison}
\label{sec:computation}

We compare the running time of {\em computing} the different metrics on TPUv4, given two sets of embeddings with sample size $n$. We emphasize that this is only the time to compute the metric, not to generate the samples, which is roughly linear in $n$. In Figure~\ref{fig:walltime-walltime}, we observe that computing \texttt{MIND} at its recommended sample size of $5$k is more than $2$ orders of magnitude faster than computing FID. This is an additional difference, on top of the time necessary to sample a much larger sample when using FID.

\subsection{Peak memory comparison}
\label{sec:memory}

We compare the peak memory required to compute the different metrics on a TPUv4, using two sets of embeddings with sample size $n$. Note that these measurements reflect only the additional memory consumed during metric computation which do not include the memory occupied by the input data itself, the latter results are provided in Appendix~\ref{app:peak-memory}. In Figure~\ref{fig:walltime-mem}, we highlight that computing \texttt{MIND} at its recommended sample size ($n=5$k) requires over an order of magnitude less memory than computing either MMD or FID. (Note that the curves for MMD and $\MIND$ collapse across different input dimensions, resulting in overlapping lines).

\begin{figure}
\begin{subfigure}{.48\textwidth}
\centering
\includegraphics[width=\linewidth]{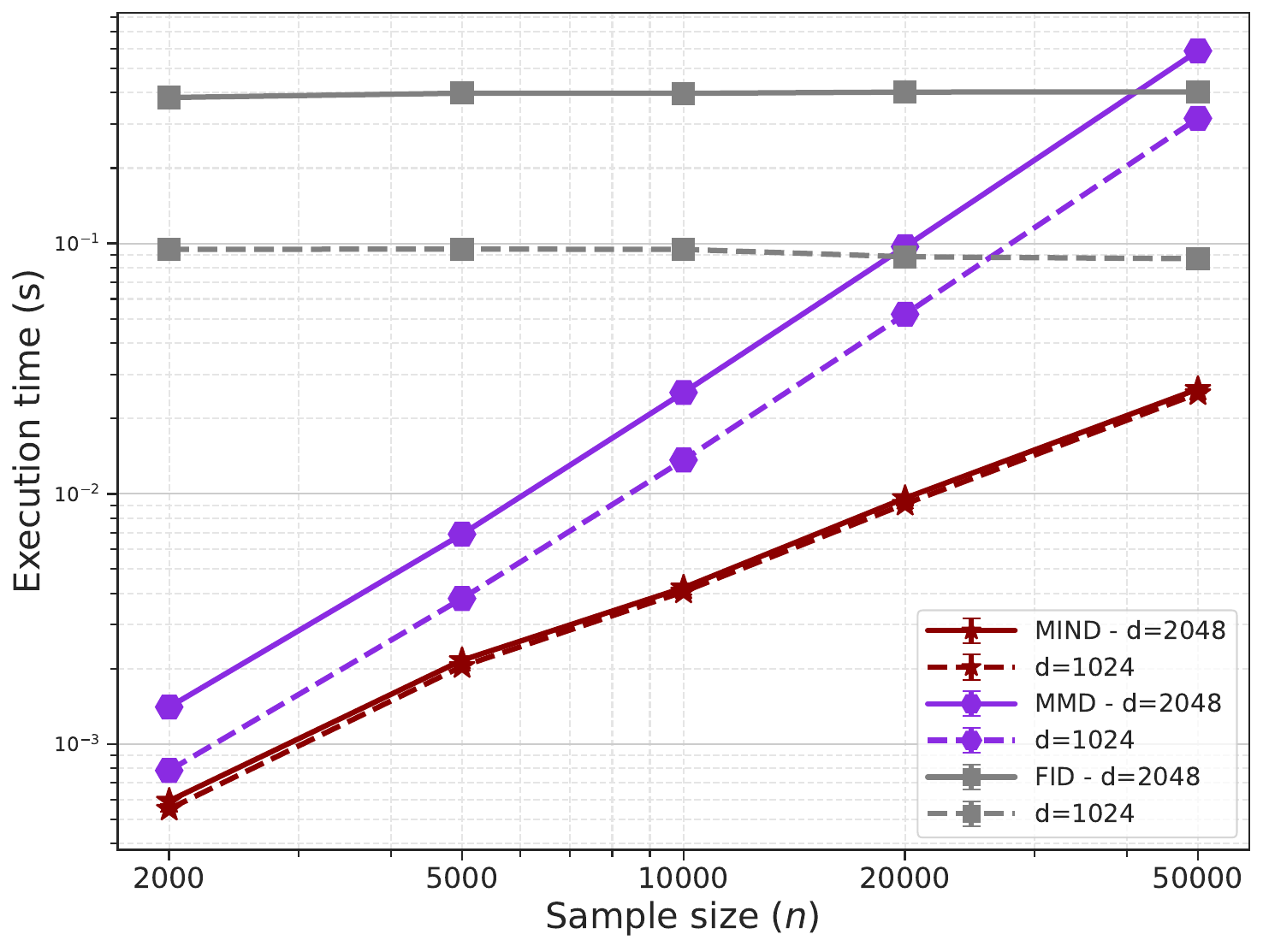}
\subcaption{Execution time}
\label{fig:walltime-walltime}
\end{subfigure}
\begin{subfigure}{.48\textwidth}
\centering
\includegraphics[width=\linewidth]{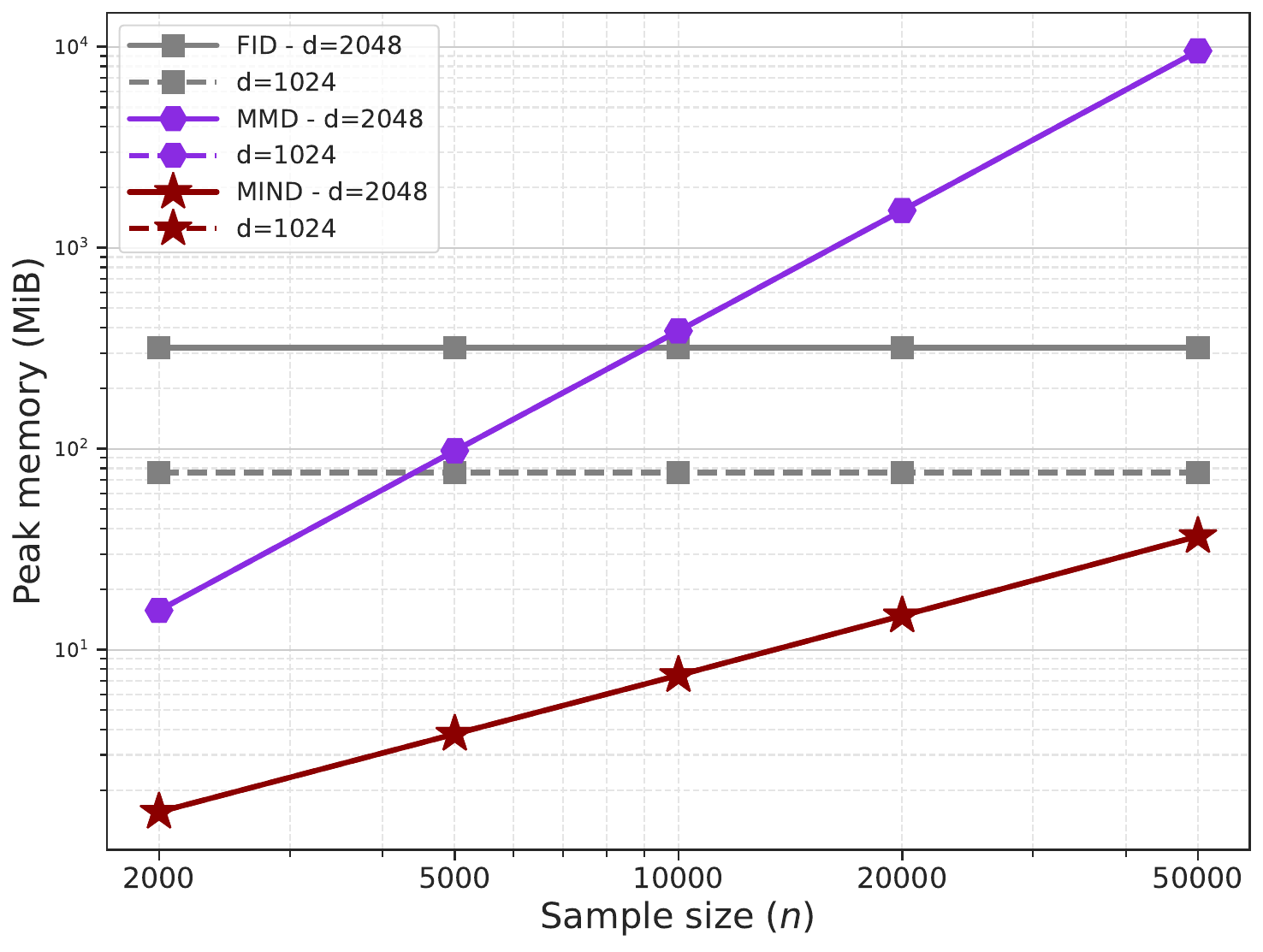}
\subcaption{Peak Memory}
\label{fig:walltime-mem}
\end{subfigure}
\cprotect\caption{Walltime and peak memory comparison for $\MIND$, MMD, and FID  \label{fig:walltime}}
\end{figure}

\section*{Conclusion}
In this work, we introduced \texttt{MIND}, a metric for evaluating generative models that addresses statistical and computational limitations of FID. Our empirical results demonstrate that \texttt{MIND} is faster to compute and achieves stable evaluations with sample sizes as low as $2$k, compared to the $50$k typically required for FID. Furthermore, as a proper distance metric, \texttt{MIND} exhibits better robustness to moment-matching adversarial attacks than other metrics, while being affected by it. As a purely statistical metric, \texttt{MIND} measures the distributional distance to a reference dataset (such as the training data). It is not designed to evaluate other qualitative aspects of the generated images, such as visual aesthetics or text legibility. We believe this metric provides a rigorous, efficient, and reliable standard for assessing the quality of modern generative models.

\newpage
\bibliographystyle{apalike}
\bibliography{ref}

@article{Dudley1969TheSO,
  title={The Speed of Mean {Glivenko}-{Cantelli} Convergence},
  author={Richard M. Dudley},
  journal={Annals of Mathematical Statistics},
  year={1969},
  volume={40},
  pages={40-50},
  url={https://api.semanticscholar.org/CorpusID:121048606}
}

@inproceedings{radford2021learning,
  title={Learning transferable visual models from natural language supervision},
  author={Radford, Alec and Kim, Jong Wook and Hallacy, Chris and Ramesh, Aditya and Goh, Gabriel and Agarwal, Sandhini and Sastry, Girish and Askell, Amanda and Mishkin, Pamela and Clark, Jack and others},
  booktitle={International conference on machine learning},
  pages={8748--8763},
  year={2021},
  organization={PMLR}
}

@article{simeoni2025dinov3,
  title={{DINOv3}},
  author={Sim{\'e}oni, Oriane and Vo, Huy V and Seitzer, Maximilian and Baldassarre, Federico and Oquab, Maxime and Jose, Cijo and Khalidov, Vasil and Szafraniec, Marc and Yi, Seungeun and Ramamonjisoa, Micha{\"e}l and others},
  journal={arXiv:2508.10104},
  year={2025}
}

@phdthesis{bonnotte2013unidimensional,
  title={Unidimensional and evolution methods for optimal transportation},
  author={Bonnotte, Nicolas},
  year={2013},
  school={Universit{\'e} Paris Sud-Paris XI; Scuola normale superiore (Pise, Italie)}
}

@article{yang2026representation,
  title={Representation {Fr\'echet} Loss for Visual Generation},
  author={Yang, Jiawei and Geng, Zhengyang and Ju, Xuan and Tian, Yonglong and Wang, Yue},
  journal={arXiv:2604.28190},
  year={2026}
}

@article{karras2017progressive,
  title={Progressive growing of {GAN}s for improved quality, stability, and variation},
  author={Karras, Tero and Aila, Timo and Laine, Samuli and Lehtinen, Jaakko},
  journal={arXiv:1710.10196},
  year={2017}
}

@article{bischoff2024practical,
  title={A practical guide to sample-based statistical distances for evaluating generative models in science},
  author={Bischoff, Sebastian and Darcher, Alana and Deistler, Michael and Gao, Richard and Gerken, Franziska and Gloeckler, Manuel and Haxel, Lisa and Kapoor, Jaivardhan and Lappalainen, Janne K and Macke, Jakob H and others},
  journal={arXiv:2403.12636},
  year={2024}
}

@inproceedings{
binkowski2018demystifying,
title={Demystifying {MMD} {GAN}s},
author={Mikołaj Bińkowski and Dougal J. Sutherland and Michael Arbel and Arthur Gretton},
booktitle={The Sixth International Conference on Learning Representations},
year={2018},
}

@inproceedings{nichol2021improved,
  title={Improved denoising diffusion probabilistic models},
  author={Nichol, Alexander Quinn and Dhariwal, Prafulla},
  booktitle={International conference on machine learning},
  pages={8162--8171},
  year={2021},
  organization={PMLR}
}

@inproceedings{ronneberger2015u,
  title={{U-Net}: Convolutional networks for biomedical image segmentation},
  author={Ronneberger, Olaf and Fischer, Philipp and Brox, Thomas},
  booktitle={International Conference on Medical image computing and computer-assisted intervention},
  pages={234--241},
  year={2015},
  organization={Springer}
}

@book{billingsley2017probability,
  title={Probability and measure},
  author={Billingsley, Patrick},
  year={2017},
  publisher={John Wiley \& Sons}
}

@inproceedings{chong2020effectively,
  title={Effectively unbiased {FID} and inception score and where to find them},
  author={Chong, Min Jin and Forsyth, David},
  booktitle={Proceedings of the IEEE/CVF conference on computer vision and pattern recognition},
  pages={6070--6079},
  year={2020}
}

@article{oquab2024dinov2,
title={{DINO}v2: Learning Robust Visual Features without Supervision},
author={Maxime Oquab and Timoth{\'e}e Darcet and Th{\'e}o Moutakanni and Huy V. Vo and Marc Szafraniec and Vasil Khalidov and Pierre Fernandez and Daniel Haziza and Francisco Massa and Alaaeldin El-Nouby and Mido Assran and Nicolas Ballas and Wojciech Galuba and Russell Howes and Po-Yao Huang and Shang-Wen Li and Ishan Misra and Michael Rabbat and Vasu Sharma and Gabriel Synnaeve and Hu Xu and Herve Jegou and Julien Mairal and Patrick Labatut and Armand Joulin and Piotr Bojanowski},
journal={Transactions on Machine Learning Research},
issn={2835-8856},
year={2024},
}

@inproceedings{jayasumana2024rethinking,
  title={Rethinking {FID}: Towards a better evaluation metric for image generation},
  author={Jayasumana, Sadeep and Ramalingam, Srikumar and Veit, Andreas and Glasner, Daniel and Chakrabarti, Ayan and Kumar, Sanjiv},
  booktitle={Proceedings of the IEEE/CVF Conference on Computer Vision and Pattern Recognition},
  pages={9307--9315},
  year={2024}
}

@book{cover1999elements,
  title={Elements of information theory},
  author={Cover, Thomas M},
  year={1999},
  publisher={John Wiley \& Sons}
}

@article{bradbury2018jax,
  title={{JAX}: composable transformations of {Python}+ {NumPy} programs},
  author={Bradbury, James and Frostig, Roy and Hawkins, Peter and Johnson, Matthew James and Leary, Chris and Maclaurin, Dougal and Necula, George and Paszke, Adam and VanderPlas, Jake and Wanderman-Milne, Skye and others},
  year={2018}
}

@inproceedings{stein2023exposing,
 author = {Stein, George and Cresswell, Jesse and Hosseinzadeh, Rasa and Sui, Yi and Ross, Brendan and Villecroze, Valentin and Liu, Zhaoyan and Caterini, Anthony L and Taylor, Eric and Loaiza-Ganem, Gabriel},
 booktitle = {Advances in Neural Information Processing Systems},
 editor = {A. Oh and T. Naumann and A. Globerson and K. Saenko and M. Hardt and S. Levine},
 pages = {3732--3784},
 publisher = {Curran Associates, Inc.},
 title = {Exposing flaws of generative model evaluation metrics and their unfair treatment of diffusion models},
 volume = {36},
 year = {2023}
}

@article{gretton2012kernel,
  title={A kernel two-sample test},
  author={Gretton, Arthur and Borgwardt, Karsten M and Rasch, Malte J and Sch{\"o}lkopf, Bernhard and Smola, Alexander},
  journal={Journal of Machine Learning Research},
  volume={13},
  number={25},
  pages={723--773},
  year={2012},
  publisher={JMLR. org}
}

@inproceedings{szegedy2016rethinking,
  title={Rethinking the Inception Architecture for Computer Vision},
  author={Szegedy, Christian and Vanhoucke, Vincent and Ioffe, Sergey and Shlens, Jon and Wojna, Zbigniew},
  booktitle={2016 IEEE Conference on Computer Vision and Pattern Recognition (CVPR)},
  pages={2818--2826},
  year={2016},
  organization={IEEE}
}

@book{villani2008optimal,
  title={Optimal transport: old and new},
  author={Villani, C{\'e}dric},
  volume={338},
  year={2008},
  publisher={Springer}
}

@article{peyre2019computational,
  title={Computational optimal transport: With applications to data science},
  author={Peyr{\'e}, Gabriel and Cuturi, Marco},
  journal={Foundations and Trends{\textregistered} in Machine Learning},
  volume={11},
  number={5-6},
  pages={355--607},
  year={2019},
  publisher={Now Publishers, Inc.}
}

@inproceedings{rabin2011wasserstein,
  title={Wasserstein barycenter and its application to texture mixing},
  author={Rabin, Julien and Peyr{\'e}, Gabriel and Delon, Julie and Bernot, Marc},
  booktitle={International conference on scale space and variational methods in computer vision},
  pages={435--446},
  year={2011},
  organization={Springer}
}

@phdthesis{nadjahi:tel-03533097,
  TITLE = {{Sliced-Wasserstein distance for large-scale machine learning : theory, methodology and extensions}},
  AUTHOR = {Nadjahi, Kimia},
  URL = {https://theses.hal.science/tel-03533097},
  NUMBER = {2021IPPAT050},
  SCHOOL = {{Institut Polytechnique de Paris}},
  YEAR = {2021},
  MONTH = Nov,
  TYPE = {Theses},
  HAL_ID = {tel-03533097},
  HAL_VERSION = {v1},
}

@INPROCEEDINGS {deng2009imagenet,
author = { Deng, Jia and Dong, Wei and Socher, Richard and Li, Li-Jia and Kai Li and Li Fei-Fei },
booktitle = { 2009 IEEE Computer Society Conference on Computer Vision and Pattern Recognition Workshops (CVPR Workshops) },
title = {{ ImageNet: A large-scale hierarchical image database }},
year = {2009},
volume = {},
ISSN = {1063-6919},
pages = {248-255},
publisher = {IEEE Computer Society},
address = {Los Alamitos, CA, USA},
month =Jun}

@inproceedings{liu2015faceattributes,
  author={Liu, Ziwei and Luo, Ping and Wang, Xiaogang and Tang, Xiaoou},
  booktitle={2015 IEEE International Conference on Computer Vision (ICCV)}, 
  title={Deep Learning Face Attributes in the Wild}, 
  year={2015},
  volume={},
  number={},
  pages={3730-3738}
}

@inproceedings{paszke2019pytorch,
 author = {Paszke, Adam and Gross, Sam and Massa, Francisco and Lerer, Adam and Bradbury, James and Chanan, Gregory and Killeen, Trevor and Lin, Zeming and Gimelshein, Natalia and Antiga, Luca and Desmaison, Alban and Kopf, Andreas and Yang, Edward and DeVito, Zachary and Raison, Martin and Tejani, Alykhan and Chilamkurthy, Sasank and Steiner, Benoit and Fang, Lu and Bai, Junjie and Chintala, Soumith},
 booktitle = {Advances in Neural Information Processing Systems},
 editor = {H. Wallach and H. Larochelle and A. Beygelzimer and F. d\textquotesingle Alch\'{e}-Buc and E. Fox and R. Garnett},
 pages = {},
 publisher = {Curran Associates, Inc.},
 title = {{PyTorch}: An Imperative Style, High-Performance Deep Learning Library},
 volume = {32},
 year = {2019}
}

@inproceedings{nadjahi2020statistical,
  author = {Nadjahi, Kimia and Durmus, Alain and Chizat, L\'{e}na\"{\i}c and Kolouri, Soheil and Shahrampour, Shahin and Simsekli, Umut},
  booktitle = {Advances in Neural Information Processing Systems},
  editor = {H. Larochelle and M. Ranzato and R. Hadsell and M.F. Balcan and H. Lin},
  pages = {20802--20812},
  publisher = {Curran Associates, Inc.},
  title = {Statistical and Topological Properties of Sliced Probability Divergences},
  volume = {33},
  year = {2020}
}

@inproceedings{heusel2017gans,
 author = {Heusel, Martin and Ramsauer, Hubert and Unterthiner, Thomas and Nessler, Bernhard and Hochreiter, Sepp},
 booktitle = {Advances in Neural Information Processing Systems},
 editor = {I. Guyon and U. Von Luxburg and S. Bengio and H. Wallach and R. Fergus and S. Vishwanathan and R. Garnett},
 pages = {},
 publisher = {Curran Associates, Inc.},
 title = {{GAN}s Trained by a Two Time-Scale Update Rule Converge to a Local Nash Equilibrium},
 volume = {30},
 year = {2017}
}

@InProceedings{pmlr-v84-genevay18a,
  title = 	 {Learning Generative Models with {Sinkhorn} Divergences},
  author = 	 {Genevay, Aude and Peyre, Gabriel and Cuturi, Marco},
  booktitle = 	 {Proceedings of the Twenty-First International Conference on Artificial Intelligence and Statistics},
  pages = 	 {1608--1617},
  year = 	 {2018},
  editor = 	 {Storkey, Amos and Perez-Cruz, Fernando},
  volume = 	 {84},
  series = 	 {Proceedings of Machine Learning Research},
  month = 	 {09--11 Apr},
  publisher =    {PMLR},
}

@inproceedings{cuturi2013sinkhorn,
 author = {Cuturi, Marco},
 booktitle = {Advances in Neural Information Processing Systems},
 editor = {C.J. Burges and L. Bottou and M. Welling and Z. Ghahramani and K. Weinberger},
 pages = {},
 publisher = {Curran Associates, Inc.},
 title = {Sinkhorn Distances: Lightspeed Computation of Optimal Transport},
 volume = {26},
 year = {2013}
}

@inproceedings{diffusion-ddpm,
 author = {Ho, Jonathan and Jain, Ajay and Abbeel, Pieter},
 booktitle = {Advances in Neural Information Processing Systems},
 editor = {H. Larochelle and M. Ranzato and R. Hadsell and M.F. Balcan and H. Lin},
 pages = {6840--6851},
 publisher = {Curran Associates, Inc.},
 title = {Denoising Diffusion Probabilistic Models},
 volume = {33},
 year = {2020}
}

@inproceedings{inception-score,
 author = {Salimans, Tim and Goodfellow, Ian and Zaremba, Wojciech and Cheung, Vicki and Radford, Alec and Chen, Xi},
 booktitle = {Advances in Neural Information Processing Systems},
 editor = {D. Lee and M. Sugiyama and U. Luxburg and I. Guyon and R. Garnett},
 pages = {},
 publisher = {Curran Associates, Inc.},
 title = {Improved Techniques for Training {GAN}s},
 url = {https://proceedings.neurips.cc/paper_files/paper/2016/file/8a3363abe792db2d8761d6403605aeb7-Paper.pdf},
 volume = {29},
 year = {2016}
}

@inproceedings{fid,
  title     = {{GAN}s Trained by a Two Time-Scale Update Rule Converge to a Local Nash Equilibrium},
  author    = {Heusel, Martin and Ramsauer, Hubert and Unterthiner, Thomas and Nessler, Bernhard and Hochreiter, Sepp},
  booktitle = {Advances in Neural Information Processing Systems (NIPS)},
  year      = {2017}
}

@inproceedings{fid-critique-pr,
 author = {Sajjadi, Mehdi S. M. and Bachem, Olivier and Lucic, Mario and Bousquet, Olivier and Gelly, Sylvain},
 booktitle = {Advances in Neural Information Processing Systems},
 editor = {S. Bengio and H. Wallach and H. Larochelle and K. Grauman and N. Cesa-Bianchi and R. Garnett},
 pages = {},
 publisher = {Curran Associates, Inc.},
 title = {Assessing Generative Models via Precision and Recall},
 volume = {31},
 year = {2018}
}

@article{monge-1781,
  title   = {M{\'e}moire sur la th{\'e}orie des d{\'e}blais et des remblais},
  author  = {Monge, Gaspard},
  journal = {Histoire de l'Acad{\'e}mie Royale des Sciences},
  year    = {1781},
  pages   = {666--704}
}

@article{kantorovich-1942,
  title   = {On the translocation of masses},
  author  = {Kantorovich, Leonid V.},
  journal = {Doklady Akademii Nauk SSSR},
  volume  = {37},
  number  = {7-8},
  pages   = {227--229},
  year    = {1942}
}

\newpage

\appendix

\section{Definitions}
\label{app:defs}
\subsection{Empirical measures}
\begin{definition}
For a sample $Y_1, \ldots, Y_n$ of size $n$ from some data distribution $\pdata$, we denote by $\hat p_{n, \text{data}}$ the {\em empirical distribution} of the $Y_i$s, defined by
\[
    \hat p_{n, \text{data}} = \frac{1}{n}\sum_{j} \delta_{Y_j}\, .
\]
Similarly, for $X_1, \ldots, X_n$ from $p_\theta$ we denote by $\hat p_{n, \theta}$ the empirical distribution of the $X_i$s
\[
    \hat p_{n, \theta} = \frac{1}{n}\sum_{j} \delta_{X_j}\, .
\]
\end{definition}

\subsection{Optimal transport}
\label{app:ot}

\begin{definition}
The optimal transport problem for Euclidean cost, also called the 2-Wasserstein distance is defined for two probability distributions $p, q \in \cP_2(\R^d)$ with finite second moments as
\begin{align}\label{eq:ot-samples}
    W_2^2(p, q) &= \min_{T: T_\# p = q}\E_{X \sim p}[\|X - T(X)\|^2]\\
    &= \min_{\pi \in \Pi(p,q)} \E_\pi[\|X-Y\|^2]\, ,
\end{align}
The first definition is defined as the {\em Monge} formulation \citep{monge-1781}, and the second one as the {\em Kantorovitch} formulation \citep{kantorovich-1942}, with the equivalence holding when $p, q$ are absolutely continuous, or discrete uniform samples of the same finite size.
\end{definition}
Note that this distance can be approximated with sample access to $p$ and $q$ by plugging directly these samples empirical measures $\hat p_n$ of the $X_j$ and $\hat q_n$ of the $Y_j$. However, this approach suffers from two issues: It suffers from a curse of dimensionality, and the convergence of the estimate $W_2^2(\hat p_n, \hat q_n)$ to $W_2^2(p, q)$ is slow, in $n^{-1/d}$ \citep{Dudley1969TheSO}, it is slow to compute in general, with a worst case super cubic cost of $n^3$ for the Hungarian algorithm, and methods based on the Sinkhorn algorithm in $n^2 \log(n)$. The latter is motivated by an entropic-regularized formulation \citep{cuturi2013sinkhorn} 
\[
W^2_{2, \varepsilon}= \min_{\pi \in \Pi(p,q)} \E_\pi[\|X-Y\|^2] - \varepsilon H(\pi)\, .
\]

An interesting exception is the one-dimensional case: when $d=1$, the solution of \eqref{eq:ot-samples} is given by
\[
W_2^2(\hat p_n, \hat q_n) = \frac{1}{n} \sum_{j=1}^n |\text{sort}(x)_j - \text{sort}(y)_j|^2\,
\]
where $\text{sort}: \R^n \to \R^n$ is the function that maps a vector $x\in \R^n$ to its copy sorted in nondecreasing order $\text{sort}(x) = (x_{\sigma(1)},\ldots,x_{\sigma(n)})^\top$ such that $x_{\sigma(1)} \le \ldots \le x_{\sigma(n)}$ (ties do not make this function ambiguous). It can therefore be computed in time of order $n \log n$. This can be leveraged for $d>1$ by considering the average Wasserstein distance over uniformly random unit directions. This is called the {\em sliced Wasserstein distance}

\begin{definition}
The sliced Wasserstein distance \citep{rabin2011wasserstein, nadjahi:tel-03533097} is defined as the average of the Wasserstein distances over 1-d projections along $u \sim \cU(S)$ a uniformly random unit direction
\[
SW_2^2(p, q) = \E_{u \sim \cU(S)}[W_2^2(u^\top p, u^\top q)]\,
\]
where $u^\top p$ (resp. $u^\top q$) denotes the distribution of $u^\top X$ when $X \sim p$ (resp. $u^\top Y$ when $Y \sim q$).

\end{definition}

The sliced Wasserstein distance is still a distance between distributions. It can also be easily estimated from samples, given empirical measures $\hat p_n$ and $
\hat q_n$ and $M$ i.i.d. unit vectors $u_1, \ldots u_M$
\begin{align*}
\hat{SW}^2_{2, M}(\hat p_n, \hat q_n) &= \frac{1}{M} \sum_{i=1}^M W_2^2(u_i^\top \hat p_n, u_i^\top \hat q_n) \\
&= \frac{1}{Mn}\sum_{i=1}^M \sum_{j=1}^n |\text{sort}(u_i^\top X)_j -  \text{sort}(u_i^\top Y)_j|^2 \, ,
\end{align*}
One of the advantages of this approach is the relaxed computational load: computing Wasserstein distances in 1D only requires to sort all the elements in the sample, which can be done in order of $n \log n$ time, and is highly parallelizable, allowing to perform $M$ of these operations with little to no overhead, and $M$ projections from dimension $d$ to 1, for $n$ points each time.

In particular, under mild assumptions, the sample complexity of estimating the sliced Wasserstein distance does not depend on the dimension of the problem \citep{nadjahi2020statistical}, in contrast to the standard Wasserstein distance for which the sample complexity grows exponentially with the dimension. 

We finally note that as in \cite{rabin2011wasserstein}, we consider the average of the {\em squared} 1-D Wasserstein distances, and would do so for other $\ell_p$, $p\neq 2$ norm costs.

\subsection{Metric comparison}
\paragraph{Remarks about FID}
\begin{itemize}[topsep=0pt,itemsep=2pt,parsep=2pt,leftmargin=10pt]
    \item[-] The last formula is also found in the literature as the following, both are equal
\begin{align*}
 &W_2^2(\cN(\mu_X, \Sigma_X), \cN(\mu_Y, \Sigma_Y)) = \|\mu_X - \mu_Y\|^2 \\
 & \quad + \text{tr}(\Sigma_X + \Sigma_Y - 2(\Sigma_Y^{1/2}\Sigma_X\Sigma_Y^{1/2})^{1/2})\, .
\end{align*}

    \item[-] In practice, the expectations are obtained based on a finite sample $X_1, \ldots, X_n$ from a generative model and $Y_1, \ldots, Y_n$ from a dataset, and we actually compute the plug-in estimate 
    \begin{align*}
    \FID(\hat p_{n, \theta}, \hat p_{n, \text{data}}) &= \|\hat \mu_X - \hat \mu_Y\|^2 \\
    & \quad + \text{tr}(\hat \Sigma_X + \hat \Sigma_Y - 2(\hat \Sigma_Y \hat \Sigma_X)^{1/2})\, .
    \end{align*}
    
    \item [-] This distance is motivated by the {\em Wasserstein distance} $W_2^2$ \citep[see, e.g.][and Appendix~\ref{app:ot} in this work]{villani2008optimal}, obtained by solving an optimal transport problem, with a square Euclidean distance cost. For FID, this distance is applied to two fitted Gaussian distributions rather than to the sample distributions $\hat p_{n, \theta}$ and $\hat p_{n, \text{data}}$.\\

    \item[-] Using this method, rather than a sample-based estimate of  $W_2^2(\hat p_{n, \theta}, \hat p_{n, \text{data}})$, allows to overcome the two main obstacles when using the Wasserstein distance between two distributions based on sample access: {\em statistical} and {\em computational complexity}. Computing the FID only requires to estimate the mean and covariance matrices, and to perform a conceptually simple, closed-form computation. 
\end{itemize}

\label{app:metrics}
\begin{definition}[mean FID]
    Let $X = \psi_w(g_\theta(Z)) \sim p_\theta$ and $Y = \psi_w(D) \sim \pdata$, 
\[
\mu_X = \E_{p_\theta}[X] \,, \quad \mu_Y = \E_{\pdata}[Y] \, .
\]

The mean FID (that we denote by $\mu \FID$) is defined as 
\[
\mu\FID(p_\theta, \pdata) = \|\mu_X - \mu_Y\|^2\, .
\]
\end{definition}
\paragraph{Remarks}
\begin{itemize}[topsep=0pt,itemsep=2pt,parsep=2pt,leftmargin=10pt]

\item[-] Much like for FID, it is also very easy to estimate the mean FID from a finite sample with the plug-in empirical measures $\mu\FID(\hat p_{n, \theta}, \hat p_{n, \text{data}}) = \|\hat\mu_X - \hat\mu_Y\|^2$.\\

\item[-] We show in Section~\ref{sec:exp-sample} that the sample complexity of $\mu \FID$ is much lower than that of $\FID$ - this probably stems from the fact that only a vector of size $d$ must be evaluated rather than a $d$-by-$d$ matrix.\\

\item[-] We show in Section~\ref{sec:exp-hacking} that it  is even less robust than $\FID$.
\end{itemize}

\begin{definition}[Sliced FID]
Let $X = \psi_w(g_\theta(Z)) \sim p_\theta$ and \mbox{$Y = \psi_w(D) \sim \pdata$}, the sliced FID (that we denote by $\sigma \FID$) is defined as
\[
\sigma \FID(p_\theta, \pdata) = \E_{u\sim\cU(S)} [\FID(u^\top p_\theta, u^\top \pdata)]\, .
\]
For finite samples $(X_j)_{j\in[n]}$, $(Y_j)_{j\in[n]}$ and $(u_i)_{i \in [M]}$ it can be estimated by
\begin{align*}
\sigma \FID(\hat p_{n, \theta}, \hat p_{n, \text{data}}) &= \frac{1}{M} \sum_{i=1}^M \FID(u_i^\top \hat p_n, u_i^\top \hat q_n) \\
&= \frac{1}{M}\sum_{i=1}^M\big\{(u_i^\top \hat \mu_{n, X} - u_i^\top \hat \mu_{n, Y})^2 \\
&\quad + (\hat \sigma_{n, u_i^\top X} - \hat \sigma_{n, u_i^\top Y})^2\big\} \, .
\end{align*}
\end{definition}
\paragraph{Remarks}
\begin{itemize}[topsep=0pt,itemsep=2pt,parsep=2pt,leftmargin=10pt]

    \item[-] While very easy to estimate from samples, we also show that it suffers from the same robustness issues as the FID.
\end{itemize}

\begin{definition}[Sinkhorn divergence \citep{pmlr-v84-genevay18a}] For two distributions, we denote by $W_{\varepsilon}(p, q)$ the value of the entropic-regularized optimal transport problem between $p$ and $q$ \citep[see, e.g.][and Appendix~\ref{app:ot} in this work]{peyre2019computational}. Let $X = \psi_w(g_\theta(Z)) \sim p_\theta$ and \mbox{$Y = \psi_w(D) \sim \pdata$}, the Sinkhorn Divergence Inception Distance (SDID) is defined by
\begin{align*}
\text{SDID}_\varepsilon(p_\theta, \pdata) = W_{\varepsilon}(p_\theta, \pdata) - \frac{1}{2}W_{\varepsilon}(p_\theta, p_\theta) - \frac{1}{2}W_{\varepsilon}(\pdata, \pdata)\, .
\end{align*}
\end{definition}
\paragraph{Remarks}
\begin{itemize}[topsep=0pt,itemsep=2pt,parsep=2pt,leftmargin=10pt]

    \item[-] For finite samples, the empirical measures $\hat p_{n, \theta}, \hat p_{n, \text{data}}$ can be split in $\hat p_{1,n, \theta}, \hat p_{2,n, \theta}, \hat p_{1, n, \text{data}}, \hat p_{2, n, \text{data}}$ and the divergence can be estimated by
    \begin{align*}
    \text{SDID}_\varepsilon(\hat p_{n, \theta}, \hat p_{n, \text{data}}) &= W_{\varepsilon}(\hat p_{1,n, \theta}, \hat p_{1, n, \text{data}}) \\
    &\quad - \frac{1}{2}W_{\varepsilon}(\hat p_{1,n, \theta}, \hat p_{2,n, \theta}) \\
    &\quad - \frac{1}{2}W_{\varepsilon}(\hat p_{1, n, \text{data}}, \hat p_{2, n, \text{data}})\, .
    \end{align*}
SDID can be computed by computing each entropic regularized optimal transport problem with a fast GPU-friendly alternate projection method, called Sinkhorn's algorithm \citep{cuturi2013sinkhorn}.\\

\item[-] In practice, to overcome a curse of dimensionality, we have found it better to estimate the correction term from two independent samples $\hat p_{1,n}, \hat p_{2,n}$. This concretely doubles the required sample size. We have found this metric to be much more robust than FID, and to require a smaller sample size, but of a similar order (ignoring this doubling) - see Section~\ref{sec:exp-sample}. 
\end{itemize}

\begin{definition}[Maximum mean discrepancy - MMD \citep{gretton2012kernel, jayasumana2024rethinking}]
    Let $X = \psi_w(g_\theta(Z)) \sim p_\theta$ and $Y = \psi_w(D) \sim \pdata$, and the kernel function $k_\sigma(x,y) = \exp(-\|x-y\|^2/\sigma)$, for $\sigma>0$, the MMD is defined as
\begin{align*}
\text{MMD}(p_\theta, \pdata) &= \E_{p_\theta\otimes p_\theta}[k(x, x')]- 2 \E_{p_\theta\otimes \pdata}[k(x, y)] \\
&\quad + \E_{\pdata \otimes \pdata}[k(y, y')]\, .
\end{align*}
\end{definition}
\paragraph{Remarks}
\begin{itemize}[topsep=0pt,itemsep=2pt,parsep=2pt,leftmargin=10pt]

    \item[-] Since it is defined as a two-sample mean, the MMD can also be estimated quickly from empirical distributions $\hat p_{n, \theta}$ and $\hat p_{n, \text{data}}$.\\

    \item[-] One of the drawbacks of this metric is the need to select a hyperparameter $\sigma>0$.\\

    \item[-] Another drawback is the computational aspect, as an $n \times n$ kernel matrix must be computed. \\

    \item[-] The Sinkhorn divergence and MMD are related: when $\varepsilon \to +\infty$, we have that 
    \[
    \text{SDID}_\varepsilon \to \frac{1}{2} \text{MMD}_{-\|\cdot\|^2}\,
    \]
     where the kernel function $k$ is given by the negative squared Euclidean distance (rather than a Gaussian kernel).\\
     
     \item[-] The use of this metric, with another embedding network, is recommended in \citep{jayasumana2024rethinking}.
\end{itemize}

\section{Proofs}
\label{app:proofs}

\subsection{Proof of Proposition 4.1}
\label{app:proof}
\begin{proof}
Recall that the target distribution $p$ has mean $\mu$ and covariance $\Sigma$. We construct the discrete distribution $\hat{q}$ using $2r$ vectors $v_i^{(+)} = \mu + \alpha u_i$ and $v_i^{(-)} = \mu - \alpha u_i$, where each vector is assigned probability $\pi_i = \lambda_i / (2\tr(\Sigma))$, and $\alpha = \sqrt{\tr(\Sigma)}$.

\paragraph{1. Mean of $\hat{q}$}
By the definition of expectation for a discrete distribution:
\begin{align}
\E_{\hat{q}}[v] &= \sum_{i=1}^r \pi_i v_i^{(+)} + \sum_{i=1}^r \pi_i v_i^{(-)} \\
&= \sum_{i=1}^r \pi_i (\mu + \alpha u_i) + \sum_{i=1}^r \pi_i (\mu - \alpha u_i) \\
&= \sum_{i=1}^r \pi_i (2\mu) = \mu \sum_{i=1}^r 2\pi_i.
\end{align}
Since $2\pi_i = \lambda_i / \tr(\Sigma)$ and $\sum_{i=1}^r \lambda_i = \tr(\Sigma)$, it follows that $\sum 2\pi_i = 1$, hence $\E_{\hat{q}}[v] = \mu$.

\paragraph{2. Covariance of $\hat{q}$}
The covariance of $\hat{q}$ is given by $\E_{\hat{q}}[(v - \mu)(v - \mu)^\top]$:
\begin{align}
\text{Cov}_{\hat{q}}(v) &= \sum_{i=1}^r \pi_i (v_i^{(+)} - \mu)(v_i^{(+)} - \mu)^\top + \sum_{i=1}^r \pi_i (v_i^{(-)} - \mu)(v_i^{(-)} - \mu)^\top \\
&= \sum_{i=1}^r \pi_i (\alpha u_i)(\alpha u_i)^\top + \sum_{i=1}^r \pi_i (-\alpha u_i)(-\alpha u_i)^\top \\
&= \sum_{i=1}^r 2\pi_i \alpha^2 u_i u_i^\top.
\end{align}
Substituting $\alpha^2 = \tr(\Sigma)$ and $2\pi_i = \lambda_i / \tr(\Sigma)$:
\begin{align}
\text{Cov}_{\hat{q}}(v) &= \sum_{i=1}^r \left( \frac{\lambda_i}{\tr(\Sigma)} \right) \tr(\Sigma) u_i u_i^\top \\
&= \sum_{i=1}^r \lambda_i u_i u_i^\top = \Sigma.
\end{align}

\paragraph{3. FID Value}
The $\FID$ between two distributions is defined as the 2-Wasserstein distance between their associated Gaussians \citep{heusel2017gans}. Consequently, $\FID$ is strictly a function of the first two moments. Since $\E_{\hat{q}}[v] = \mu_p$ and $\text{Cov}_{\hat{q}}(v) = \Sigma_p$, the means and covariances match exactly:
\begin{align}
\FID(\hat{q}, p) &= \|\mu - \mu\|^2 + \tr(\Sigma + \Sigma - 2(\Sigma \Sigma)^{1/2}) \\
&= 0 + \tr(2\Sigma - 2\Sigma) = 0.
\end{align}
This concludes the proof.
\end{proof}

\newpage

\section{Additional results}
\label{app:results}
\begin{figure}
\hfill
\begin{subfigure}{.44\textwidth}
\centering
\includegraphics[width=\linewidth]{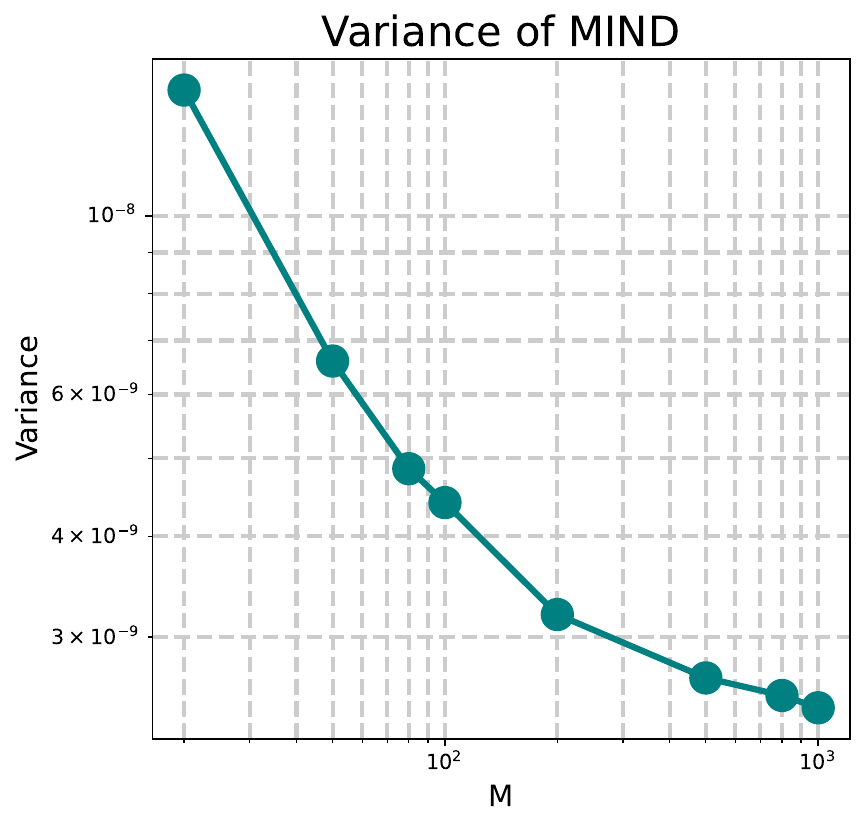}
\subcaption{Variance of the $\MIND$ with different number of projections $M$.}
\label{fig:variance}
\end{subfigure}
\hfill
\begin{subfigure}{.48\textwidth}
\centering
\includegraphics[width=\linewidth]{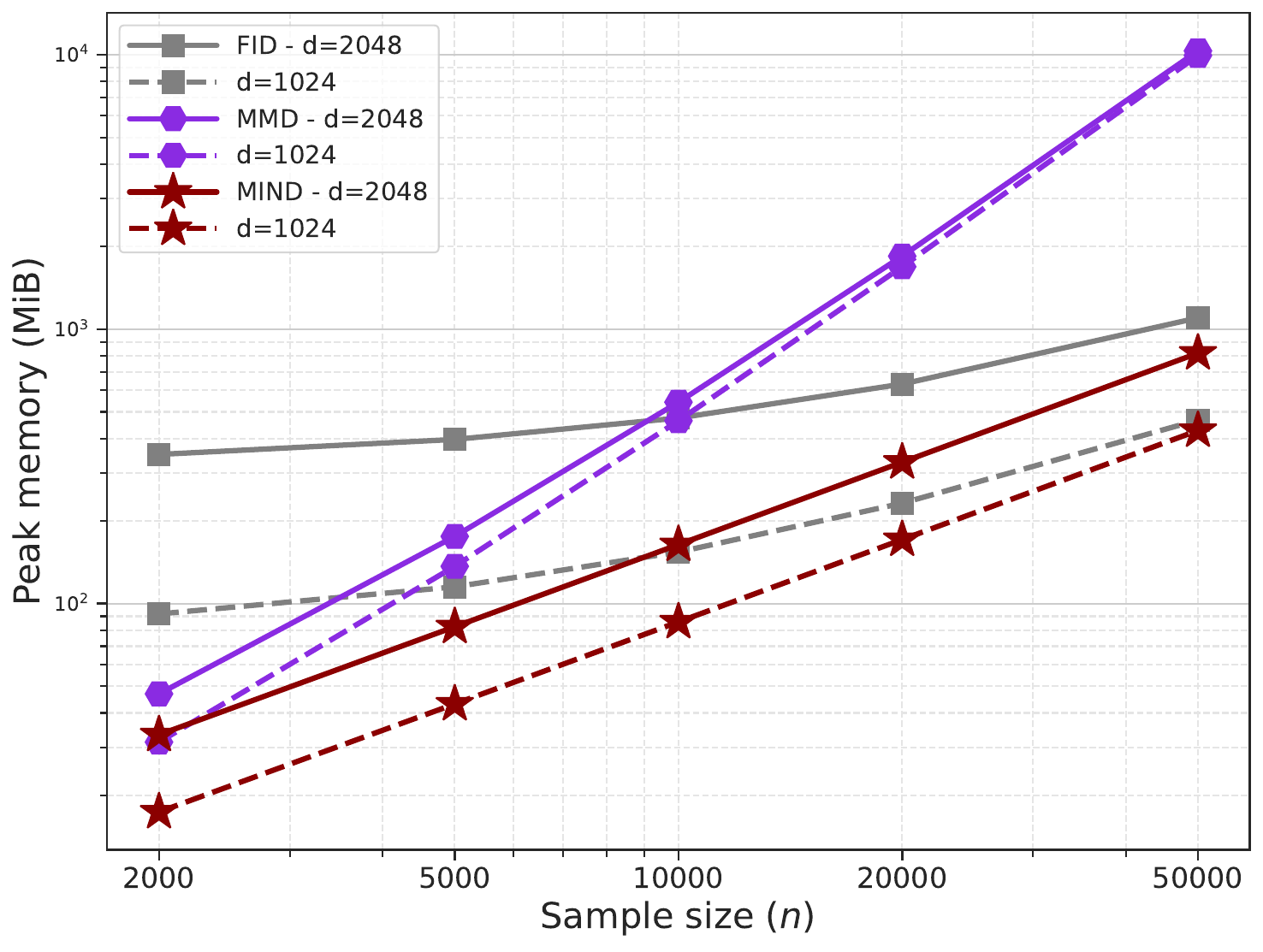}
\subcaption{Peak memory used for calculating different metrics.}
\label{fig:peak-memory}
\end{subfigure}
\hfill\hfill
\end{figure}
\subsection{Effects of number of projections}
\label{app:num-projections}
As shown in Figure~\ref{fig:variance}, our empirical analysis shows that the variance is not affected at smaller scales than numerical artifacts for $M>1000$. We also show that using $M=100$ yields almost the same performance, while there is a substantial degradation when using $M=10$.
\subsection{Peak Memory}
\label{app:peak-memory}
The measurements include both memory occupied by the input data and the temporary memory required for metric computation. As shown in Figure~\ref{fig:peak-memory}, $\MIND$ at its recommended sample size ($n=5$k) requires less memory than MMD and FID.

\subsection{Computation resources and experimental details}
\label{app:resource}
We run each experiment in Section \ref{sec:exp-sample} for $2$ hours using $4$ TPUv$5$e and each experiment in Section \ref{sec:exp-hacking} for $10$ minutes using $4$ TPUv$5$e. The diffusion model used in Section \ref{sec:exp-sample} is trained with $5$M steps on ImageNet-64, we summarize other details for its training and sampling in Table~\ref{tab:hyperparams}.
\begin{table}[ht]
    \centering
    \begin{tabular}{lc}
    \toprule
    {\bf Name} & {\bf Value} \\
    \midrule
    Condition embedding dimension & 512\\
    Noise embedding dimension & 512\\
    Optimizer & Adam with standard hyperparameters \\
    Learning Rate & $10^{-5}$ \\
    EMA decay & $0.9999$ \\
    Hardware & 16 TPUv6e \\
    \midrule
    Noise schedule & Rectified Flow \\
    Number of sampling steps & $250$ \\
    CFG weight & 0.0 \\
    \midrule
    Number of base channels & 192 \\
    Attention Head Dimension & 64 \\
    Number of downsampling & 4 \\
    Channels multiplier & (1, 2, 3, 4) \\
    Residual blocks per level & (3, 3, 3, 3) \\
    \bottomrule
    \end{tabular}
    \vspace{.5em}
    \cprotect\caption{Hyperparameters for training and sampling from diffusion models.}
    \label{tab:hyperparams}
\end{table}
\subsection{Moment-matching hacking}
\label{app:hack}
We illustrate the results of the experiment described in Section~\ref{sec:exp-hacking} in Figure~\ref{fig:hack-vis}
\begin{figure}

\centering
\includegraphics[width=0.95\linewidth]{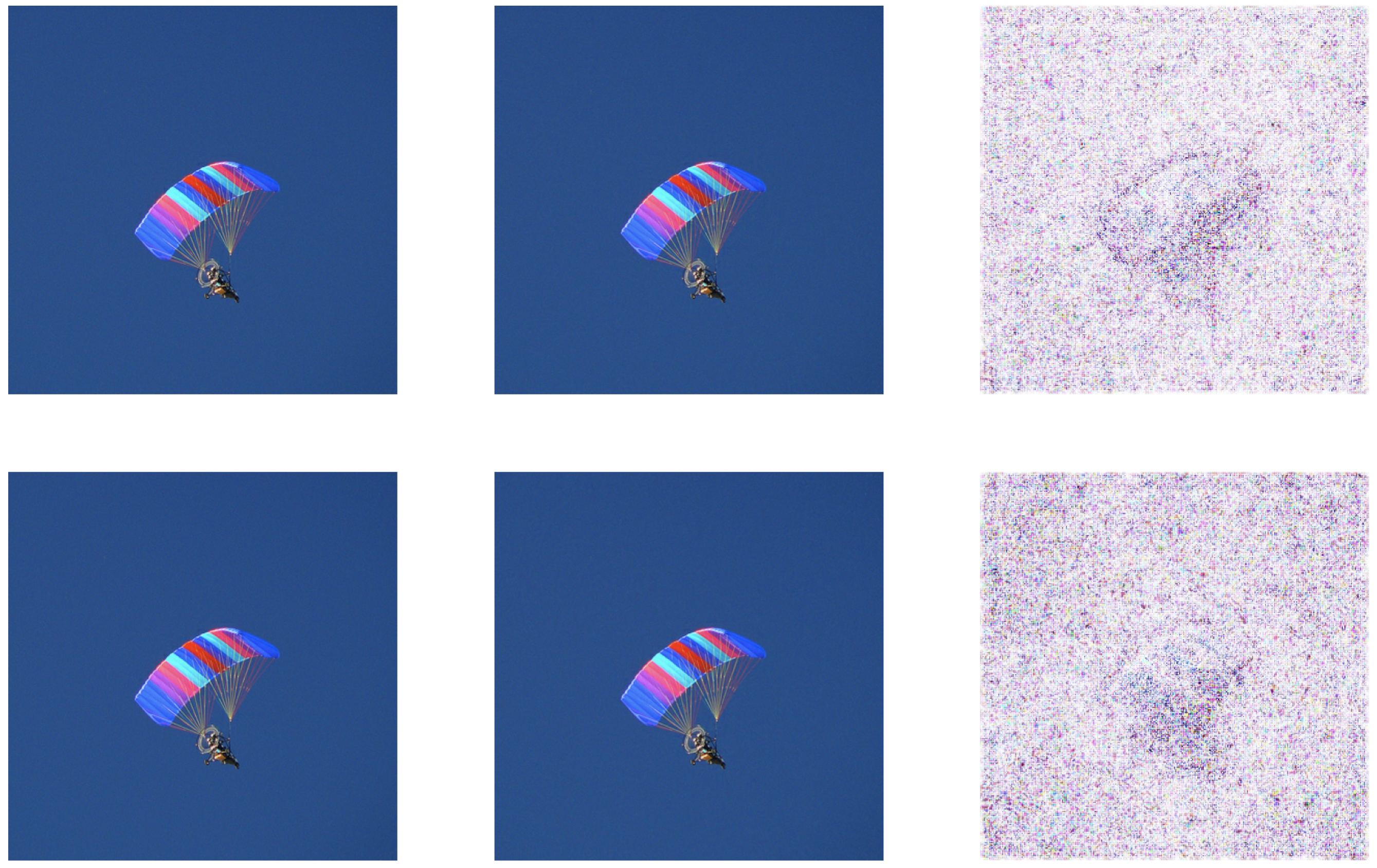}
\caption{Two elements of the batch, all initial images are the same. (\textbf{Left}) Initial image (\textbf{Center}) Image after optimization (\textbf{Right}) Difference scaled by a factor 100 (to become visible). \label{fig:hack-vis}}

\end{figure}

\end{document}